\documentclass[journal]{IEEEtran} \onecolumn 
\ifCLASSINFOpdf
\else
\usepackage[dvips]{graphicx}
\fi
\usepackage{url}
\usepackage{graphicx}
\usepackage[colorlinks=true, allcolors=blue]{hyperref}
\usepackage[utf8]{inputenc} 
\usepackage[T1]{fontenc}    
\usepackage{hyperref}       
\usepackage{url}            
\usepackage{booktabs}       
\usepackage{amsfonts}       
\usepackage{nicefrac}       
\usepackage{microtype}      
\usepackage{xcolor}         
\usepackage{amsmath}
\usepackage[square,sort,comma,numbers]{natbib}
\usepackage{algorithm} 
\usepackage{algpseudocode}
\usepackage[noend]{algcompatible}\usepackage{amsmath}
\usepackage{hyperref}
\usepackage{amsmath, mathtools, amsfonts, bm, amssymb, amsthm,graphicx, caption, subcaption,multirow}
\usepackage{algorithm} 
\usepackage{algpseudocode}
\usepackage{bm, bbm}
\usepackage{color}
\usepackage{comment}
\usepackage{adjustbox}
\DeclareMathOperator*{\argmin}{\mathrm{argmin}}

\renewcommand{\u}{\bm{u}}
\newcommand{\ustar}{\u^*{}}

\newcommand{\x}{\bm{x}}
\newcommand{\g}{\bm{g}}
\newcommand{\G}{\bm{G}}
\newcommand{\X}{{\bm{X}}}

\renewcommand{\S}{\bm{S}}

\newcommand{\e}{\bm{e}}
\newcommand{\y}{\bm{y}}
\newcommand{\w}{\bm{w}}
\newcommand{\U}{{\bm{U}}}

\newcommand{\calU}{{\mathcal U}}
\newcommand{\Qopt}{\Q_{*,00}}

\newcommand{\R}{{\bm{R}}}
\newcommand{\M}{\bm{M}}

\newcommand{\I}{\bm{I}}
\newcommand{\Y}{\bm{Y}}

\newcommand{\Z}{{\bm{{Z}}}}
\newcommand{\D}{{\bm{D}}}

\newcommand{\Q}{{\bm{Q}}}
\newcommand{\F}{{\bm{F}}}
\newcommand{\iidsim}{\stackrel{\mathrm{iid}}{\thicksim }}

\newcommand{\SE}{\mathrm{SD}}  

\newcommand{\A}{\bm{A}}

\newcommand{\gradU}{\mathrm{GradU}} 
\newcommand{\gradUj}{\mathrm{Grad}{\u^j}} 

\newcommand{\svdeq}{\overset{\mathrm{SVD}}=} 
\newcommand{\qreq}{\overset{\mathrm{QR}}=} 
\newcommand{\bi}{\begin{itemize}} \newcommand{\ei}{\end{itemize}}
\newcommand{\ben}{\begin{enumerate}} \newcommand{\een}{\end{enumerate}}

\newcommand{\vsm}{\vspace{-0.1in}}

\renewcommand{\b}{\bm{b}}
\renewcommand{\d}{\bm{d}}
\newcommand{\B}{\bm{B}}
\newcommand{\V}{{\bm{V}}}
\newcommand{\W}{{\bm{W}}}

\newcommand{\z}{\bm{z}}

\newcommand{\Bstar}{{\B^\star}}   
\newcommand{\bstar}{\b^\star}             

\newcommand{\Ustar}{\U^\star{}}
\newcommand{\Xstar}{\X^\star}
\newcommand{\xstar}{\x^\star}

\newcommand{\Vstar}{\V^\star{}}

\newcommand{\sigmin}{{\sigma_{\min}^*}}
\newcommand{\sigmax}{{\sigma_{\max}^*}}
\newcommand{\sigminTwo}{{\sigma_{\min}^{*2}}}
\newcommand{\sigmaxTwo}{{\sigma_{\max}^{*2}}}
\newcommand{\sigmaxFour}{{\sigma_{\max}^{*4}}}

\renewcommand{\P}{\bm{P}_{*,\perp}}

\newcommand{\E}{\mathbb{E}}

\renewcommand{\V}{\bm{V}}

\renewcommand{\v}{\bm{v}}

\newcommand{\bea}{\begin{eqnarray}}
\newcommand{\eea}{\end{eqnarray}}

\newtheorem{theorem}{Theorem}[section]

\newtheorem{lemma}[theorem]{Lemma}
\newtheorem{claim}[theorem]{Claim}
\newtheorem{corollary}[theorem]{Corollary}

\newtheorem{remark}[theorem]{Remark}

\newtheorem{assu}[theorem]{Assumption}
\renewcommand\thetheorem{\arabic{section}.\arabic{theorem}}

\newcommand{\SEF}{\SE_F}

\newcommand{\eps}{\epsilon}

\newcommand{\sigmamin}{\sigma_{\min}}

\newcommand{\jk}{{jk}}
\newcommand{\deltajk}{\xi_\jk} 
\renewcommand{\b}{\bm{b}}
\renewcommand{\forall}{\text{ for all }}
\renewcommand{\Re}{\mathbb{R}}

\newcommand{\lV}{\lVert} \newcommand{\rV}{\rVert}

\newcommand{\m}{\bm{m}}
\newcommand{\deltat}{\delta^{(t)}}
\newcommand{\deltatold}{\delta^{(t-1)}}
\newcommand{\deltazero}{\delta^{(0)}}

\title{Efficient Federated Low Rank Matrix Completion}
\author{{Ahmed Ali Abbasi and Namrata Vaswani}
\thanks{The authors are at the Department of Electrical and Computer Engineering, Iowa State University. Email: namrata@iastate.edu.
A part of this work (only simulation experiments) appeared in Allerton 2023 \cite{allrtnPpr}.}}

\begin{document}

\maketitle

\vsm \vsm \vsm \vsm 

\begin{abstract}
In this work, we develop and analyze a Gradient Descent (GD) based solution, called Alternating GD and Minimization (AltGDmin), for efficiently solving the low rank matrix completion (LRMC) in a federated setting. LRMC involves recovering an $n \times q$ rank-$r$ matrix $\Xstar$ from a subset of its entries when  $r \ll \min(n,q)$. Our  theoretical guarantees (iteration and sample complexity bounds) imply that AltGDmin is the most communication-efficient solution in a federated setting, is one of the fastest, and  has the second best sample complexity among all iterative solutions to LRMC. In addition, we also prove two important corollaries. (a) We provide a guarantee for AltGDmin for solving the noisy LRMC problem. (b)  We show how our lemmas can be used to provide an improved sample complexity guarantee for AltMin, which is the fastest centralized solution.
\end{abstract}

\IEEEpeerreviewmaketitle

\section{Introduction}
In this work, we develop and analyze a fast and communication-efficient solution for efficiently solving the low rank matrix completion (LRMC) problem in a federated setting. LRMC involves recovering/learning a low rank (LR) matrix from a small number of its observed entries. Our proposed algorithm, and the sample and iteration complexity guarantees that we prove for it, are also useful in centralized settings.
Federation means that (i) different subsets of the data are acquired at different distributed nodes; and (ii) all nodes can only communicate with the central node or ``center''.  Communication-efficiency is a key concern with all distributed algorithms, including federated ones. Privacy of the data is another concern in a federated setting. In this work, ``privacy'' means the following: the nodes' raw data (here observed matrix entries) cannot be shared with the center; and the algorithm should be such that the center cannot recover any of the columns of the LR matrix.


\subsection{Problem Setup and Notation}
LRMC involves recovering a rank-$r$ matrix $\Xstar \in \Re^{n \times q}$, where $r \ll \min(n,q)$, from a subset of its entries. Entry $j$ of column $k$, denoted $\Xstar_\jk$, is observed, independently of all other observations, with probability $p$. Let $\xi_\jk \iidsim \mathrm{Bernoulli}(p)$ for $j \in [n], k \in [q]$. Then, the set of observed entries, denoted by $\Omega$, is
\[
\Omega := \{(j,k): \xi_\jk = 1\}
\]
By setting the unobserved entries to zero, the observed data matrix $\Y \in \Re^{n \times q}$ can be defined as
\begin{equation}
\Y_\jk := \begin{cases}
\Xstar_\jk &   \text{ if } (j,k) \in \Omega, \\ 
0 & \text{ otherwise}.
\end{cases}
\ \ \text{or, equivalently,} \ \
\Y := \Xstar_\Omega
\label{eq:msrmnts}
\end{equation}
Here and below, $\M_{\Omega}$ refers to the matrix $\M$ with all entries whose indices are not in the set $\Omega$ are zeroed out; while the rest of the entries remain unchanged.

\subsubsection{Notation}
For any matrix $\M$, $\m_k$ denotes its $k$-th column while $\m^j$ denotes its $j$-th row transposed (so it is a column vector). $(\cdot)^\intercal$ denotes the matrix/vector transpose. We use $\I$ to denote the identity matrix and $\e_k$ to denote its $k$-th column (this is a 1-0 vector with 1 at the $k$-th location and zero everywhere else).  $\| \cdot \|$ denotes either the vector $\ell_2$ or the induced $\ell_2$ norm  of a matrix ($\|\M\|:= \max_{\z:\|\z\|=1}\|\M \z \| = \sigma_{\max}(\M)$),  while $\| \cdot \|_F$ denotes the Frobenius norm of a matrix. We use $vec(\M)$ to denote a vectorized version of matrix $\M$; thus $\|\M\|_F = \|vec(\M)\|$.
We use the standard notation $[q]:=\{1,2,\dots, q\}$.

For a tall matrix $\M$, $\M^\dagger \triangleq (\M^\intercal\M)^{-1} \M^\intercal$ denotes the Moore-Penrose pseudo-inverse. Also, $\mathrm{QR}(\M)$  maps the matrix $\M \in \Re^{n \times r}$ to $\Q \in \Re^{n \times r}$ such that $\M = \Q\R$ is the QR decomposition of $\M$ (columns of $\Q$ are orthornormal).
For matrices  $\U_1,\U_2$ with orthonormal columns, the two commonly used measures of subspace distance (SD) between their column spans are
$
\SE_2(\U_1,\U_2) := \| (\I - \U_1\U_1^\intercal)\U_2 \|
$
and
$
\SE_F(\U_1,\U_2) := \| (\I - \U_1\U_1^\intercal)\U_2 \|_F.
$
We use the Frobenius norm SD in most of our analysis in this paper because it allows us to get a better sample complexity guarantee. 

We use $\M = \mathrm{diag}( m_k, k \in [q])$ to denote a diagonal $q \times q$ matrix with scalar entries $m_k$ and $\M = \mathrm{blockdiag}( \M_k,  k \in [q])$ to denote a block-diagonal matrix with block diagonal entries $\M_k$.
We use $\Omega_k := \{j \in [n] \mid  \xi_\jk = 1\}$ to denote the set of indices of the observed entries in column $k$ and we define a diagonal 1-0 matrix  $\S_k \in \Re^{n \times n}$ as
\[
\S_k:= \mathrm{diag}(\xi_\jk, \ j \in [n])
\]
Thus, $\y_k = \S_k \xstar_k$ for all $k \in [q]$.
{\em We reuse the letters $c, C$ to denote different numerical constants in each use with the convention that $c < 1$ and $C \ge 1$.}

Let $\Xstar \svdeq \Ustar (\bm{\Sigma^*}) {\V^*}:= \Ustar \Bstar$, where $\Ustar \in \Re^{n \times r}$ and has orthonormal columns and $\Vstar \in \Re^{r \times q}$ with orthonormal rows. We use $\kappa= \sigmax/\sigmin$ to denote the condition number of the diagonal $r \times r$ matrix $\bf\Sigma^*$. Here $\sigmax,\sigmin$ to denote its largest, smallest singular values.  Also, we let $\Bstar:=\bm{\Sigma^*} {\V^*}$ so that $\Xstar = \Ustar \Bstar$.



%



\subsubsection{Assumption}
As in all past works on LRMC, e.g. \cite{matcomp_candes,lowrank_altmin}, we need the following assumption on the singular vectors of $\Xstar$; this is a way to guarantee that the rows and columns of $\Xstar$ are dense (non-sparse). This and the  LR assumption help ensure that one can correctly interpolate (fill in) the missing entries even with observing only a few entries of each row or column.

\begin{assu}[$\mu$-incoherence of singular vectors of $\Xstar$]  
Assume  row norm bounds on $\Ustar$: $\max_{j \in [n]}\| \ustar{}^j\| \leq \mu \sqrt{r/n}$, and 
column norm  bounds on $\V^*$: $\max_{k \in [q]} \|  \v^*_k \| \leq \mu \sqrt{r/q}$. 
Since $\Bstar= \bm{\Sigma^*} {\V^*}$, this implies that $\| \bstar_k \| \le \mu \sqrt{r/q} \sigmax$.
\label{incoh}
\end{assu}

We should point out that our proof for AltGDmin correctness only uses the incoherence of $\Ustar$ and of $\Bstar$. This is a slightly weaker requirement than requiring incoherence of $\Ustar,\Vstar$.

\subsubsection{Federation} We assume that there are a total of $\gamma$ nodes, with $\gamma \le q$.
Each node has access to a different subset of the columns of the observed data matrix $\Y$. We use $\S_\ell$ to denote the subset of columns of $\Y$ available at node $\ell$. The sets $\S_\ell$ form a partition of $[q]$, i.e., they are mutually disjoint and $\cup_{\ell=1}^\gamma \S_\ell = [q]$. To keep notation simple, we assume $q$ is a multiple of $\gamma$ and $|\S_\ell| = q/\gamma$.
All nodes can only communicate with a central node or ``center''. In a federated setting, the two desirable properties are communication-efficiency and ``privacy''. {\em In this work, ``privacy'' means the following. The nodes' raw data cannot be shared with the center and the center should not be able reconstruct the matrix $\Xstar$.} 

Our discussion of complexities assumes $\gamma \ll q$ and treats $\gamma$ as a numerical constant. Thus order $|\Omega|/\gamma$ is equal to order $|\Omega|$ with $|\Omega|\ge (n+q) r$ (the number of samples needs to be larger than the number of unknowns in rank $r$ matrix).

\subsection{Related Work} \label{subsec:relatedWork}
\begin{table*}[t!]
	\begin{center}
		{
			\begin{tabular}{llll} \toprule
				Algorithm &  Computation Comp.  	& Communic. Comp.  & Sample Comp.     \\
				\hline \midrule
				AltGDMin (Priv) (Proposed) & $\kappa^2 \frac{|\Omega|}{\gamma}r^2 \log (\frac{1}{\epsilon})$   & $\kappa^2 nr \log (\frac{1}{\epsilon})$  &$\kappa^6 \mu^2 nr^2\log n \log(\frac{1}{\eps})$    \\
				&&& \\
				FactGD (Priv)   \cite{rpca_gd,lafferty_lrmc} & $\kappa \mu  \frac{|\Omega|}{\gamma}r^2 \log (\frac{1}{\epsilon})$    & $\kappa \mu nr^2 \log (\frac{1}{\epsilon})$   &$\kappa^4 \mu^2 nr^2\log n$      \\
				&&& \\
				AltMin (Priv) \cite{lowrank_altmin}  &  $ \frac{|\Omega|}{\gamma}r \log^2(\frac{1}{\epsilon})$   &  $nr \log^2(\frac{1}{\eps})$  & $\kappa^4 \mu^2 nr^{4.5}\log n\log(\frac{1}{\eps})$   \\
				&&& \\
				AltMin (Not-Priv) \cite{lowrank_altmin}  & $\frac{|\Omega|}{\gamma}r^2  \log (\frac{1}{\epsilon})$ & $ \frac{|\Omega|}{\gamma}  \log (\frac{1}{\epsilon})$  & $\kappa^4 \mu^2 nr^{4.5}\log n \log(\frac{1}{\eps})$   \\
				&&& \\
				Smooth-AltMin (Not-Priv) \cite{hardt2014understanding}  &  $\frac{|\Omega|}{\gamma}r^2 \log n \log(\frac{1}{\eps})$   &   $ \frac{|\Omega|}{\gamma}  \log n \log(\frac{1}{\eps})$ & $\kappa^2 \mu^2 nr^{3}\log n \log(\frac{1}{\eps})$   \\
				&&& \\
				ProjGD (Not-Priv) \cite{fastmc,rmc_gd} & $\mu^4 \frac{|\Omega|}{\gamma}r \log^2 (\frac{1}{\eps})$    & $\mu^4  \frac{|\Omega|}{\gamma} \log^2(\frac{n}{\epsilon}) $ & $\mu^4 nr^2 \log^2 n \log^2(\frac{1}{\eps})$    \\
				&&& \\
				\hline
			\end{tabular}
		}
		\vspace{-0.05in}
		\caption{\small\sl
			The table assumes $n \approx q$, $\gamma$ is a numerical constant, $\max(\log(1/\eps), r) = \log(1/\eps)$, and $|\Omega| \ge nr$ (necessary).  Here  $\text{Communic Comp} = T \cdot \max(\text{Communic.(node), Communic.(center)}) $. 
			Similarly for the computation cost. 
		}
		\label{tab1}
	\end{center}
	\vspace{-0.15in}
\end{table*}
Starting with the seminal work of \cite{matcomp_first,matcomp_candes} which introduced a nuclear norm based convex relaxation, the LRMC problem has been extensively studied in the last decade and a half \cite{matcomp_first,matcomp_candes,optspace,lowrank_altmin,mc_luo,rmc_gd,pr_mc_reuse_meas, lowrank_altmin_no_kappa,lafferty_lrmc,rpca_gd}. Two classes of algorithms feature prominently in this literature - solutions to convex relaxations and direct iterative algorithms. The former \cite{matcomp_first,matcomp_candes} are slow: the required number of iterations for $\eps$ accuracy (iteration complexity) grows as $1/\sqrt{\epsilon}$ \cite{lowrank_altmin}.
The first provably accurate iterative solution was the  Alternating Minimization (AltMin) algorithm with a spectral initialization \cite{lowrank_altmin,optspace}. AltMin was shown to converge geometrically (iteration complexity order $\log(1/\eps)$) with a sample complexity of $\kappa^4 \mu^2 nr^{4.5} \log(1/\epsilon)$ in \cite{lowrank_altmin}. Subsequent work  \cite{hardt2014understanding} considered a modified version of AltMin and improved its sample complexity to order $\kappa^2 \mu^2 nr^{2.5}\log(1/\epsilon)$. AltMin factorizes the unknown LR matrix $\X$ as $\X = \U\B$. After initializing $\U$, it alternately updates $\B$ and $\U$ using minimization for one, keeping the other fixed.
Later works proposed two gradient descent (GD) based algorithms -  Projected GD (ProjGD) \cite{fastmc,rmc_gd} and Factorized GD (FactGD) \cite{rpca_gd,lafferty_lrmc} - that reduced the sample complexity dependence on $r$ to $r^2$. ProjGD involves  GD, followed by projection onto the space of rank $r$ matrices (this is done by singular value decomposition (SVD)) after each GD iteration.
FactGD factorizes the unknown LR matrix $\X$ as $\X = \U\B$, where $\U, \B$ have $r$ columns and rows respectively, and updates both by GD as follows. After initialization, at each iteration, it updates $\U$ and $\B$ by one GD step for the cost function  $f(\U,\B) + c_1 \|\U^\intercal \U - \B\B^\intercal  \|_F^2$, followed by projecting each of them onto the set of matrices with incoherent rows and columns respectively. Denote the former (incoherent rows) set by $\calU$. Initialization involves SVD on $\Y$ followed by a similar projection onto $\calU$.
The second term of the FactGD cost function is a norm balancing term that ensures the norm of $\U$ does not keep increasing with iterations while that of $\B$ decreases (or vice versa).
%
The best ProjGD guarantee needs a sample complexity of $\mu^4 nr^2 \log^2 n \log^2 1/\epsilon$ while FactGD reduces this even further to $\kappa^4 \mu^2 nr^2 \log^2 n $.  ProjGD converges geometrically with a constant GD step size, while FactGD needs the GD step size to be of order $1/r$. Consequently, its iteration complexity is worse than that of ProjGD or AltMin by a factor of $r$. However, in terms of per-iteration time cost, ProjGD is significantly slower.  Consequently, overall in terms of total cost, it is slower. Numerically, it is much slower than all other methods, because each ProjGD iteration involves a rank $r$ SVD step (constants per iteration are much larger); see Fig 1b,1d of \cite{allrtnPpr}.
Some of this discussion treats $\kappa,\mu$ as numerical constants. 
%

To our best knowledge, {\em there is no existing work on provably accurate federated LRMC}. Federation requires a communication-efficient and private solution, with  ``private'' as defined above.

In tangentially related work \cite{dist_lrmc,dist_lrmc_2}, distributed-computing solutions to LRMC are studied. These do not consider the federated setting, instead these assume that all data is observed centrally and then is distributed to different nodes to parallelize the computation; and these develop an approximate solution to the convex relaxation which is known to be very slow.
Fully decentralized LRMC has been studied in \cite{qling_etal, qling_etal2,mardani2013decentralized}. These methods are much slower than GD.
Other tangentially related works include \cite{fed_lrmc_1,fed_lrmc_2,fed_lrmc_3_byz}, all these consider differential privacy or attack resilience.  


The Alternating GD and Minimization (AltGDmin) algorithm was introduced in our past work \cite{lrpr_gdmin,lrpr_gdmin_2} as a fast solution to the LR column-wise compressive sensing (LRCS) problem. LRCS involves recovering $\Xstar$ from $\y_k:=\A_k \xstar_k$ when $\A_k$'s are $m \times n$ i.i.d. random Gaussian matrices (each entry of each $\A_k$ is i.i.d. standard Gaussian), and the right singular vectors of $\Xstar$ satisfy the incoherence assumption.
The overall idea of AltGDmin is as follows \cite{lrpr_gdmin}. Suppose the goal is to solve $\min_{\Z} f(\Z)$.
We split the unknown variable $\Z$ into two parts, $\Z= \{\Z_a, \Z_b\}$.  After initializing $\Z_a$, we alternatively update $\Z_b,\Z_a$ using minimization for $\Z_b$ (keeping $\Z_a$ fixed at its previous value) and  GD for $\Z_a$  (keeping $\Z_b$ fixed at its previous value). We choose the split-up so that $\Z_b$ is the set of variables for which the minimization can be ``decoupled'', i.e., subsets of $\Z_b$ are functions of only a subsets of the data matrix $\Y$.
AltGDmin is an efficient solution only for problems where this is possible to do. As will be clear later, this is what helps ensure both the communication, and computation, efficiency.

A small part of this work -- only the algorithm and simulation experiments -- appeared in Allerton 2023 \cite{allrtnPpr}. Some of the time and communication complexity discussion provided there was incorrect. 




\subsection{Contributions and Novelty}
We develop and analyze the AltGDmin algorithm for solving the LRMC problem. The design of AltGDmin is motivated by a federated setting.
To our best knowledge, our work provides the first theoretical guarantees for solving LRMC in a federated setting; see Theorem  \ref{thrm}. The sample and iteration complexity bounds that we prove are applicable in centralized settings as well. In addition, we also prove two important corollaries. (a) We provide a guarantee for AltGDmin for solving the noisy LRMC problem. (b) We also show how our lemmas can be used to provide an improved sample complexity guarantee for AltMin which is the fastest centralized solution for LRMC.

Using our results we can argue that, in a federated setting, AltGDmin is the most communication-efficient solution. It is also one of the two fastest private solutions and has the second smallest sample complexity. Our new corollary for AltMin (and smooth AltMin) proves that it also has the same sample complexity as that of AltGDmin. This discussion treats $\kappa,\mu$ and $\gamma$ (number of nodes) as numerical constants. We discuss these points in more detail in Sec. \ref{discuss} and in Table \ref{tab1}. 


\subsubsection{Proof Novelty} While AltGDmin has been studied for solving the LRCS problem, there are some important differences between the two problems that necessitate new analysis approaches. LRMC samples are row and column-wise local, while LRCS samples are global for each column. Because of this, LRMC requires assuming left and right singular vectors' incoherence and proving the same for each new iterate of $\U$ and $\B$. LRCS needed this only for $\B$. We need a new approach to show the incoherence of $\U$ at each iteration since the update of $\U$ in AltGDmin is different from that in all existing solutions to LRMC.
FactGD for LRMC \cite{rpca_gd} projects  $\U$ (and $\B$) onto the space of row (column) incoherent matrices after each GD step; this automatically ensures incoherence after each update.  AltMin \cite{lowrank_altmin} updates $\U$ by solving an LS problem (and not by GD), so the incoherence proof there is different too. ProjGD \cite{rmc_gd} implemented projected GD for $\X$ itself (does not factorize it).  There are other important proof differences too. We explain these in Sec. \ref{novelty}.

\subsection{Organization}
We develop the AltGDmin algorithm and state and explain our main theoretical guarantee, Theorem \ref{thrm}, for it next in Sec. \ref{sec:algorithm}. We provide a detailed comparison of this guarantee with those from existing works in Sec. \ref{discuss}. This section also summarizes the novel proof ideas and gives a proof outline. Theorem \ref{thrm} is proved in Sec.\ref{thrm_proof_section}. The lemmas used in this proof are proved in Sec. \ref{lemma_proofs}. The corollary for noisy LRMC is provided in Sec. \ref{noisy_lrmc_th}. We use our lemmas to prove an improved guarantee for AltMin and Smooth AltMin in Sec. \ref{altmincorolls}. Numerical experiments that demonstrate our theoretical claims are described in Sec. \ref{sims}. We conclude in Sec. \ref{conclude}.

\section{Alternating GD and Minimization (AltGDMin) Algorithm and Guarantees}\label{sec:algorithm}

\subsection{AltGDmin Algorithm}
The goal is to minimize the following squared loss cost function
\begin{equation}
\min_{\substack{\check\B,\, \check\U  \colon \check\U^\intercal \check\U = \I  }}  f(\check\U,\check\B), \ f(\check\U,\check\B):= \| (\Y - \check\U \check\B)_{\Omega}  \|_F^2 \label{eq:obj}
\end{equation}
We impose the orthornormal columns constraint on $\check\U$ as one way to ensure that the norm  of $\U$  does not keep increasing or decreasing continuously with algorithm iterations, while that of $\B$ decreases or increases.

AltGDmin proceeds as follows. We initialize $\U$ as explained below; this approach is adapted from that in FactGD \cite{rpca_gd}.
After the initialization, different from AltMin \cite{lowrank_altmin}, which used alternating exact minimization for both $\U$ and $\B$, and different from FactGD \cite{rpca_gd}, which used GD for both $\U$ and $\B$,
AltGDmin alternates between exact minimization over $\B$ and a single GD step for $\U$. The GD step is followed by an orthonormalization (QR) step. 
%
The above simple changes help ensure that the AltGDmin iteration complexity is better than that of FactGD and comparable to that of AltMin. 
The use of exact minimization for one of the variables, along with not using the projection onto incoherent matrices step (replacing it by QR), helps ensure that AltGDmin provably converges with a nearly constant step size. Because of this, the AltGDmin iteration complexity is better than that of FactGD by a factor of $r$.


\subsubsection{Initialization}
As in most previous work \cite{lowrank_altmin, rpca_gd}, the first step of our initialization is to compute the top $r$ singular vectors of $\Y$; denote the matrix formed by these singular vectors by $\U^{(00)}$. This is followed by a step to make $\U^{(00)}$ incoherent. We borrow this step from \cite{rpca_gd}. It involves projecting $\U^{(00)}$ onto the space of row incoherent matrices,
\[
{\calU}: = \{ \check\U: \|{\check\u}^j \| \le \mu \sqrt{r/n} \} 
\]
i.e., computing $\Pi_{\calU}(\U^{(00)}) = \min_{\check\U \in \calU} \|\check\U - \U^{(00)}\|_F$. It is easy to see that this projection can be obtained in closed form as follows:
\[
[\Pi_{\calU}(\M)]^j = \m^j \cdot \min\left(1, \frac{\mu \sqrt{r/n} }{\| \m^j \|} \right), \ \forall  j \in [n]
\]
In words,  if a row of $\M$ has $\ell_2$ norm that is more than the threshold $\mu \sqrt{r/n}$, then one renormalizes the row so that its norm equals the threshold. If the norm is less than this threshold, then we do not change it. Clearly this is an order $nr$ time operation.
The last step is computing a QR decomposition of $\Pi_{\calU}(\U^{(00)})$. This needs time of order $nr^2$.

\subsubsection{AltGDmin iterations}
We update $\B$ and $\U$ as follows for a total of $T$ iterations.
The minimization over $\B$ is a decoupled least squares (LS) problem since $f(\check\U,\check\B)$ decouples as $f(\check\U,\check\B) = \sum_{k \in [q]} \|\y - \S_k \check\U \check\b_k\|^2$.
We update $\b_k$ as
\begin{align*}
\b_k &= \argmin_{\check\b} \| \y - \S_k \U \check\b \|^2  = (\S_k \U)^\dagger \y_k, \ \forall k \in [q]  
\end{align*}
Let $\U_k := \S_k \U$, then $\b_k:= \U_k^\dagger \y_k$. Recall that $\M^\dagger:= (\M^\intercal \M)^{-1} \M^\intercal$.
We emphasise here that we write things as above for notational ease in both explaining the idea and the proof. The time complexity for computing $\b_k$ depends only on the sub-matrix of $\U_k$ with nonzero rows. This is of size $|\Omega_k| \times r$. The time cost of computing $\b_k$ is thus order $|\Omega_k| r^2$ for a given $k$. Hence, the total cost for all nodes is $\sum_k |\Omega_k| r^2 = |\Omega| r^2$.

We update $\U$ by one GD step followed by orthnormalization using QR, i.e.,
\begin{align*}
\tilde\U^+ = \U - \eta \nabla_\U f(\U,\B)),  \text{ and }   \U^+ = \mathrm{QR} (\tilde\U^+)   
\end{align*}
The gradient is
\begin{equation*}
\nabla_\U f(\U,\B) = 2( (\U\B)_{\Omega} - \Y) \B^\intercal
= 2 \sum_{k=1}^q (\S_k\U\b_k - \y_k) \b_k^\intercal
\end{equation*}
For the gradient computation, the time cost is $\sum_k |\Omega_k| r = |\Omega| r$. The QR step needs time of order $nr^2$.

We summarize the complete algorithm in Algorithm \ref{altgdmin}. Sample splitting (line 1) is assumed, as is common in most structured data recovery literature, e.g., \cite{lowrank_altmin,hardt2014understanding,fastmc,rmc_gd}.  In fact, as we discuss in Sec. \ref{discuss}, sample-splitting is assumed for obtaining provable guarantees for all iterative solutions for LRMC (and those for LRCS) in which one or both of the alternating steps is a minimization step.%

\subsubsection{Federated AltGDmin}
In a federated setting, at each algorithm iteration, $t$, each node $\ell=1,2,\dots, \gamma$ performs two operations i) updating $\b_k$ by the LS solution, for all $k \in \S_\ell$; and ii) computation of the partial  gradient $\sum_{k \in \S_\ell} [\nabla_{\U} f(\U,\B)]_k  = \sum_{k \in \S_\ell}  (\U_k \b_k - \y)_{\Omega_k} \b_k^\intercal$.
Only the $n \times r$ partial gradient needs to be sent to the center. The center sums the received partial gradients, implements the GD step, and
computes the QR decomposition, and broadcasts the updated $\U^+$  to all the nodes. This is used by the nodes in the next iteration.
The communication cost from nodes to center is equal to the number of nonzero entries in $\nabla_{\U} f(\U,\B)$; this is $r \cdot \max(n, \sum_{k \in \S_\ell} |\Omega_k|) = r \cdot \min(n, (|\Omega| / \gamma) )= nr$ since we assume $\gamma$ is a numerical constant and $|\Omega| \ge (n+q)r$ is a necessary lower bound for any approach to work. The center to nodes communication cost is also $nr$. The computation cost  is as explained above with $|\Omega|$ replaced by $|\Omega|/\gamma$.

The initialization can be federated using the power method applied to $\Y \Y^{\intercal}$.  
This costs order $\sum_{k \in \S_\ell} |\Omega_k| r$ per power method iteration.  
The upstream (node to center) per-node communication cost is $\min(n,\sum_{k \in \S_\ell} |\Omega_k|) r = nr$ and the downstream cost is also $nr$. The power method converges linearly and thus, for $\eps_{svd}$ accuracy, $\log(1/\eps_{svd})$ iterations are required. We prove that we need $\eps_{svd}=c/\kappa^2$.

Details of this federation as well as detailed time and communication cost discussion are provided  in \cite{allrtnPpr}.



\begin{algorithm}[t]
\caption{AltGDMin}
\begin{algorithmic}[1]
\label{algorithm}
\REQUIRE partial observations $\Y$,  rank $r$,  step size $\eta$, and number of iterations $T$
\STATE{Partition $\Y$ into $2T+1$ subsets $\Y_{\Omega^{(0)}}, \cdots, \Y_{\Omega^{(2T)}}$}
\STATE{$\U^{(00)} \gets$ top $r$ left-singular vectors of $\Y_{\Omega^{(0)}}$}
\STATE{$\M^{(0)} \gets \Pi_{\calU}(\U^{(00)})$}
\STATE{$\U^{(0)} \gets \mathrm{QR}(\M^{(0)}$}
\FOR {$t \in 1 \cdots T$} 
\STATE{$\b_k^{(t)} \gets (\U_{k}^{(t-1)})^{\dagger} \y_{\Omega_k^{(t)}} \, \forall \, k \in [q] $}
\STATE{$\tilde\U^{(t)} \gets \U^{(t-1)} - \eta  (\U^{(t-1)}\B^{(t)} - \Y)_{\Omega^{(T + t)}} \B^{(t)}{}^\intercal$}
\STATE{$\U^{(t)} \gets \mathrm{QR}(\tilde\U^{(t)})$ , i.e. $\tilde\U^{(t)} \qreq \U^{(t)} \R{(t)}$}
\ENDFOR
\STATE {\bfseries {Return} $\U, \B$}
\end{algorithmic}
\label{altgdmin}
\end{algorithm}

\subsection{AltGDmin Guarantee}
We can prove the following for AltGDmin. Our theorem statements and proofs assume $n \le q$ for simplicity. Recall from above that the per-iteration per-node computation and communication complexities of AltGDmin are $\max(nr^2,(|\Omega|/\gamma)r^2)$ and $\max(nr, \min(n, (|\Omega| / \gamma))r) = nr$. Here per-node takes the maximum of the node or center complexities.
We use this and the iteration complexity (expression for $T$) derived in the result below to provide an expression for its total per-node computation and communication complexities.

Let $p_{tot}:= 2T p$. Our proof of Theorem \ref{thrm} given next obtains a lower bound on the required $p$ in each algorithm iteration.


\begin{theorem}
\label{thrm}
Pick an $\eps < 1$.
Assume that Assumption \ref{incoh} holds and that, entries of $\Xstar$ are observed independently with probability $p_{tot}$ satisfying $n q p_{tot} > C  \kappa^6 \mu^2 q r^2 \log q \log({1}/{\epsilon})$. Set $\eta = 0.5/(p\sigmaxTwo)$ and $T = C\kappa^2 \log(1/\eps)$. Then, with probability (w.p.) at least $1 - 4T/n^3$,
\begin{equation}
\SE_F(\U^{(T)},\Ustar)  \leq \eps \text{ and } \|\X^{(T)}  - \Xstar\|_F \le \eps \|\Xstar\|.
\end{equation}
(recall that $\X^{(T)} = \U^{(T)}\B^{(T)}$).

The total per-node computation complexity of federated AltGDmin is $C\kappa^2 \log(1/\eps) \cdot \max(n,|\Omega|/\gamma)r^2 $ and its total per-node communication complexity is  $C\kappa^2 \log(1/\eps) \cdot nr$. 
\end{theorem}
The proof is given in Sec. \ref{thrm_proof_section} and \ref{lemma_proofs}. Before this, we discuss this result and the novel ideas used in our proof in Sec. \ref{discuss}.

Observe that $nq\cdot p_{tot} = \E[|\Omega|]$, i.e., it is the expected value of the sample complexity. We often just use the phrase ``sample complexity'' when referring to it in our writing. Treating $\kappa,\mu$ as numerical constants, the above result says that as long as we observe order $nr^2 \log q  \log (1/\eps)$ matrix entries, and we set the GD step size, $\eta$, and the total number of iterations, $T$, as stated, then  with high probability (whp), we can fill in the rest of the entries accurately: the normalized Frobenius norm of the error in this estimation is at most $\eps$. Also, we can estimate the column span of $\Xstar$ with $\eps$ accuracy. The number of iterations needed is order $\log(1/\eps)$.  We can achieve any $\eps$ accuracy as long as we increase the sample and iteration complexity accordingly.

\begin{remark}
We state Theorem \ref{thrm} for one value of the step size $\eta$ because this makes some of our proof arguments notationally less messy. However, our proof will go through for any step size $\eta \le 0.5/(p \sigmax^2)$. We can actually show the following: let $\eta = c_\eta /(p \sigmax^2)$ for a $c_\eta \le 0.5$. Then, if at each iteration, $nq \cdot p > C  \kappa^6 \mu^2 q r^2 \log q$, then  with probability (w.p.) at least $1 - 4/n^3$,
\[
\SE_F(\U^+, \Ustar) \le (1 - \frac{c c_\eta}{\kappa^2}) \SE_F(\U, \Ustar)
\]
\end{remark}

\section{Discussion and Proof Main Ideas/Novelty} \label{discuss}

\subsection{Discussion}

The average (expected value of) the required sample complexity is $nq p$. For this discussion, we assume $n \approx q$, $\kappa \approx \mu$, $\max(\log(1/\eps), r) = \log(1/\eps)$. 
%
Our guarantee given in Theorem \ref{thrm} shows that the AltGDmin sample complexity is $\kappa^6 \mu^2 nr^2 \log n \log(1/\eps)$. This is worse than that of FactGD by a factor of $\kappa^2\log(1/\eps)$. However, the iteration complexity of AltGDmin is $T=\kappa^2 \log(1/\eps)$ while that of FactGD is  $T=\kappa \mu r  \log(1/\eps)$. Assuming $\kappa \approx \mu$, this means the AltGDmin iteration complexity is better by a factor of $r$. This is also demonstrated numerically in Sec. \ref{sims}.
Per iteration, both have the same communication cost, while the FactGD computation cost is lower than that of AltGDmin by a factor of $r$.
Consequently, the total AltGDmin communication cost is lower than that of FactGD by a factor of $r$, while the total computation costs are similar. We provide a summary in Table \ref{tab1}.

The ProjGD sample complexity is worse than that of AltGDmin by a factor of $\log n \log(1/\eps)$. Moreover, the ProjGD communication cost per iteration is much higher than that of AltGDmin and FactGD. The same is true for its computation cost making it one of the slowest and most communication inefficient. 

The original sample complexity of AltMin and Smooth AltMin is much higher than that of AltGDmin or FactGD (see Table \ref{tab1}). However, the corollary that we prove for AltMin later in Sec \ref{altmincorolls} shows that the AltMin sample complexity is $\kappa^4 \mu^2 nr^2 \log n \log(1/\eps)$ and its iteration complexity is $\log(1/\eps)$. 

In summary, when compared with all other methods, AltGDmin has the lowest total communication cost. It also has one of the lowest total computation costs and the second lowest sample cost. Moreover as we argue in detail in \cite{allrtnPpr}, AltGDmin is private and requires just one data exchange per algorithm iteration. Federated FactGD can be made private with two data exchanges per algorithm iteration. AltMin can be made private by solving the LS problem needed to update $\U$ using GD rather than the closed form solution. This requires order $\log(1/\eps)$ extra data exchanges per main algorithm iteration. AltMin-Private is much slower both in numerical experiments and using theoretical complexities; see Table \ref{tab1} and Sec. \ref{sims}.

The reason that AltGDmin needs a  factor of $\log(1/\eps)$ in its sample complexity is because its proof uses sample-splitting. To our best knowledge, all guarantees for all solutions for LRMC or LRCS that factor $\X$ as $\X=\U\B$, and in which one (or both) of the two alternating steps is a minimization step, use sample-splitting \cite{lowrank_altmin,hardt2014understanding,lrpr_it,lrpr_best,lrpr_gdmin,lrpr_gdmin_2}. Some of this cited work provides guarantees for LR phase retrieval which is a phaseless measurements' generalization of LRCS; and hence any LRPR solution automatically solves LRCS. All three problems involve recovery from non-global measurements (no scalar measurement depends on the entire $\Xstar$). For such problems, properties such as restricted isometry property cannot be proved.%

The reason that AltGDmin iteration and sample complexity need an extra factor of $\kappa^2$  is due to the orthonormalization step (QR step) when updating $\U$. For any algorithm that updates one or both of $\U,\B$ by GD, such a step is one way to prevent the norm of the estimate $\U$ from increasing/decreasing over iterations (while that of $\B$ decreases/increases). The other way to ensure this is to add a norm balancing term in the cost function for the GD step for $\U$, as done in FactGD. However this would then also require a row norm clipping step for $\U$ at each iteration, like in FactGD. Use of these two steps to replace the QR step can possibly help improve its dependence on $\kappa$. But it will also mean that we cannot borrow the overall proof approach developed for AltGDmin for LRCS in \cite{lrpr_gdmin_2} that allows us to get a guarantee with a (nearly) constant step size, consequently, a (nearly) best iteration complexity of $\kappa^2\log(1/\eps)$. 

\subsection{Proof Novelty}\label{novelty}



\ben

\item The most important difference between LRCS and LRMC is that the LRMC proofs require incoherence of each iterate of $\U$ and of $\B$. When analyzing AltGDmin for LRCS in \cite{lrpr_gdmin,lrpr_gdmin_2}, this was needed only for $\B$ because LRCS measurements are column-wise global  (matrix $\A_k$ is dense and so each entry of $\y_k$ depends on the entire column $\xstar_k$). Our proof of incoherence of each updated $\U$ needs a different approach than that used in all past work on LRMC as well because our algorithm is different. (i) ProjGD \cite{rmc_gd} does not use a factorized representation for $\X$ and hence the gradient w.r.t. $\U$ does not exist for it. (ii) We do not update $\U$ by LS and hence we cannot use the approach used for AltMin or Smooth AltMin \cite{lowrank_altmin,hardt2014understanding}. (iii) FactGD \cite{rpca_gd} uses row norm clipping to ensure incoherence at each algorithm iteration, while AltGDmin does not use this (except for initialization).
   The details of our approach are explained below in the proof outlines, and in Lemma \ref{lem:rowIncoh} (and proof) and its use in proving Theorem \ref{thrm}.


As noted earlier, not using row norm clipping at each iteration is one reason why we can borrow the overall proof technique used for LRCS \cite{lrpr_gdmin_2} in order to prove a correctness guarantee with a constant GD step size. The other reason is that we use minimization for updating $\B$ (and not GD).


\item Unlike LRCS, LRMC measurements are both row-wise and column-wise local. Consequently, it is not possible to get a tight column-wise bound (bound on $\|\b_k - \U^\intercal \xstar_k\|$, and hence on $\|\x_k - \xstar_k\|$, for each $k$) under the desired sample complexity. We can only bound $\|\B- \U^\intercal \Xstar\|_F$ and $\|\X - \Xstar\|_F$. This has implications on how we bound the gradient norm (both the gradient expected value and its deviation from it). This is one reason why the use of $\SE_F$ (instead of $\SE_2$) in our bounding helps prove a better guarantee.

\item  We need to use the matrix Bernstein inequality \cite{versh_book} for bounding all of terms, instead of the concentration bounds used in \cite{lrpr_gdmin_2} (sub-Gausian Hoeffding or sub-exponential Bernstein inequality followed by an epsilon-net argument).

\item Our initialization guarantee is better than that of AltMin and Smooth Altmin \cite{lowrank_altmin,hardt2014understanding} by a factor of $r^{3}$ and $r$ respectively, and comparable to that of FactGD.
The reason for this is that we adapt the row norm clipping idea of \cite{rpca_gd} for our initialization step.
This approach is better than the entry-wise clipping used in \cite{lowrank_altmin}.
Row norm clipping can be analyzed in terms of $\SE_F(\U^{(0)},\Ustar)$ and this is the other reason we use $\SE_F$ in our proofs (see Sec. \ref{lemmas_proof_outline}).


\item All analysis of AltGDmin for LRCS has been for the noise-free case. In this work, we also provide an extension of our result to the more practical noisy case. Our proof explains a general strategy for doing this for any iterative algorithm.

\item Lastly, we also show how our lemmas can be used to improve the sample complexity dependence of AltMin and Smoothed AltMin on $r$ from $r^{4.5}$ and $r^3$ respectively to $r^2$.

\een



\subsection{Proof outline of Theorem \ref{thrm}}\label{thm_proof_outline}
Let $\U = \U^{(t-1)}$, $\G:=\U^\intercal \Xstar$ and $\U^+ =  \U^{(t)}$.

To prove Theorem  \ref{thrm}, (i) we obtain a bound on $\deltat := \SE_F(\U^+,\Ustar)$ in terms of $\deltatold = \SE_F(\U,\Ustar)$ that can be used to show exponential error decay, whp, under the desired sample complexity. (ii) This bound requires the initialization error $\deltazero$ to be small enough. Hence we also need to analyze the initialization step to show that this is true whp. Initialization is analyzed in Lemma \ref{lem:init} (proof outline given below).  The overall idea for (i) is borrowed from \cite{lrpr_gdmin_2} but there are differences because we need each new $\U$ to be incoherent.
We use induction. The induction step assumes a certain bound on $\deltatold$ and on the row norms of $\U$ and proves the same for its updated version, $\U^+$. 
To do this, we first use the induction assumption (which implies $\kappa^2\mu$ row-incoherence of $\U$) and the fact that $\E[\nabla_\U f(\U,\B)] = p (\X- \Xstar)\B^\intercal$ to get a deterministic bound on $\deltat$.
Next, we get high probability bounds on the terms of this bound using matrix Bernstein and appropriate linear algebra. These are obtained in the lemmas given in Sec. \ref{lemmas}.
The last step is to use the induction assumption and Lemma \ref{lem:rowIncoh} to bound the row norms of $\U^+$.

Finally, we simplify the bounds in order to show that, whp, $\deltat$ decays exponentially with $t$ as long as $\eta$ is at most $0.5/(p\sigmax^2)$ and $\deltazero \le c/\kappa^2$. The proof is in Sec. \ref{thrm_proof} and it relies on the six lemmas stated in Sec. \ref{lemmas}.


\subsection{Proof outlines of the lemmas used to prove Theorem \ref{thrm}} \label{lemmas_proof_outline} 

Initialization Lemma \ref{lem:init} uses the following ideas, many of which are borrowed from initialization step analysis of FactGD \cite{rpca_gd} (with filling in the missing details from there).
(i)
Using the results of \cite{Chen_2021} (Lemma 3.21, Theorem 3.22),
we can bound $\SE_F(\U^{(00)}, \Ustar)$.
(ii) The row norm clipping step can be interpreted as projecting its input onto a convex set, $\calU:=\{\U: \|\u^j\| \le \mu\sqrt{r/n} \}$, with the projection being in Frobenius norm.
Projection onto convex sets is non-expansive, i.e., $\|\Pi_{\calU}(\U_1) - \Pi_{\calU}(\U_2)\|_F \le \|\U_1 - \U_2\|_F$ \cite[eq (9),(10)]{nashed1968decomposition},\cite{rpca_gd}. Also, $\Pi_\calU(\Ustar \Q) = \Ustar \Q$ for any $r \times r$ unitary matrix $\Q$ (since $\Ustar$ as well as $\Ustar$ times any unitary matrix belong to $\calU$).
(iii) Let $\Qopt := \argmin_{\Q \ unitary}\|\U^{(00)} - \Ustar Q \|_F$.
By Lemmas 2.5 and 2.6 of \cite{Chen_2021}, $\SE_F(\U^{(00)}, \Ustar) \approx  \|\U^{(00)} - \Ustar \Qopt \|_F$ (upper and lower bounded by RHS with a factor of $\sqrt{2}$ for upper bound).
(iv) The above ideas help bound $\|\Pi_{\calU}(\U^{(00)}) - \Ustar \Qopt \|_F$.
In the last step, we use this bound and $\Pi_{\calU}(\U^{(00)}) \qreq  \U^{(0)} \R^{(0)}$ to bound $\SE_F(\U^{(0)},\Ustar)$. The proof is in Sec. \ref{proof_init}.

Lemma \ref{lem:BerrLmma} uses incoherence of $\U$ and $\deltatold$ to bound $\|\B-\G\|_F$ where $\G=\U^\intercal \Xstar$. This  is proved by writing $vec(\B - \G) = \F^{-1} vec(\D)$, where $\F \in \Re^{nr \times nr}$ and $\D \in \Re^{r \times q}$ are defined below in Sec. \ref{prelims},
bounding $\|\F^{-1}\|$ and $\| \D \|$ by the matrix Bernstein inequality, and using $\|\B- \G\|_F = \|vec(\B - \G)\| \le \| \F^{-1}\| \cdot \|vec(\D)\| = \| \F^{-1}\| \cdot \|\D\|_F  \le \| \F^{-1}\| \cdot  \sqrt{r} \|\D\|$. The bound on $\|\F^{-1}\|$ is borrowed from \cite{lowrank_altmin}, while the rest of our bounding is different and simpler than the approach used in  \cite{lowrank_altmin}. We use matrix Bernstein which provides a much simpler proof and the resulting bound holds with a smaller sample complexity (better dependence on $r$). The proof is in Sec. \ref{proof_Blemma}.
Lemma \ref{lem:Berr_implic} uses the $\|\B-\G\|_F$ bound to bound $\|\X-\X\|_F$, and the minimum and maximums singular values of $\B$. The proof idea is similar to that in \cite{lrpr_gdmin}. We provide the proof in Sec. \ref{proof_Blemma_implic} anyway since it is very short.

Lemma \ref{lem:BIncohLmma} uses incoherence of $\U$ to show incoherence of $\B$. It follows  by writing $\b_k = (\U_k^\intercal\U_k)^{-1}(\U_k^\intercal\Ustar_k) \bstar_k$, using the fact that $\E[\U_k^\intercal\U_k] = p \U^\intercal \U = p\I$, $\E[\U_k^\intercal\Ustar_k] = p \U^\intercal \Ustar$, $\|\U^\intercal \Ustar\| \le 1$,  and using matrix Bernstein to bound deviations of both matrices from their expected values.  This proof is also much simpler than the one given in \cite{lowrank_altmin} and the result needs a smaller sample complexity (better dependence on $r$). The proof is in Sec. \ref{subsec:bkIncoh}.

Lemma \ref{lem:Ugrad} uses the $\B-\G$ bound and its implications,  $\deltatold < c/\kappa^2$, and incoherence of $\U$ and of $\B$ to bound the gradient norm. It relies on the following ideas. $ \E[\gradU]  =   p (\X - \Xstar)\B^\intercal$. This holds since the expectation is taken with respect to an independent set of samples at each iteration (possible because of sample splitting). We bound the deviation from the expected value using matrix Bernstein inequality and the assumed bounds in the lemma assumptions. 
The proof is in Sec. \ref{subsec:gradUBnd}.

Lemma \ref{lem:rowIncoh} uses incoherence of $\B$ to show incoherence of $\U$. (i) We first bound the deviation of the $j$-th row of $\gradU$, denoted $\gradU^j$, from its expected value in terms of $\max(\|\u^j\|,\|\ustar^j\|)$. To get a high probability bound of this form, we use matrix Bernstein inequality and the fact that, for showing incoherence, we do not need the bound to contain $\deltatold$. Hence, we can use a seemingly loose bound of the form $|\x_\jk - \xstar_\jk|\le 2 \max(|\x_\jk|,|\xstar_\jk|) \le 2.2 \mu \sqrt{r/q} \sigmax \cdot \max(\|\u^j\|,\|\ustar^j\|)$. This bound is, in fact, tighter than a bound that contains $\deltatold$ for the initial iterations when $\deltatold$ is allowed to be larger than order $\sqrt{r/q}$.
(ii) Next we use  $\tilde{\u}^{j^{(t+1)}} =  {\u^j}^{(t)}  - \eta\gradUj \pm \E[\gradUj]$, $\E[\gradU^j] = p (\u^j{}^\intercal \B \B^\intercal  - \ustar^j{}^\intercal \Bstar \B^\intercal)$, and an upper bound on the GD step size $\eta$ to show the second part of this lemma:  $\|{\u^j}^{+}\|	\le  (1 - c/\kappa^2)   \|{\u^j}\| +  2 \|{\ustar{}^j}\|$.
The proof is provided in Sec. \ref{subsec:rowIncohProof}.


\section{Proving Theorem \ref{thrm}} \label{thrm_proof_section} 

\subsection{Definitions and Expressions}\label{prelims}
Let $\U \equiv \U^{(t-1)}$, $\B \equiv \B^{(t)}$, and  $\X \equiv \X^{(t)} = \U^{(t-1)} \B^{(t)}$.
Also let $\U^+ \equiv \U^{(t)}$, $\B^+ \equiv \B^{(t+1)}$. Let $\G :=  \U^\intercal \Xstar$ and $\P := \I - \Ustar \Ustar^\intercal$.
Finally, 
Let
\[
\deltat := \SE_F(\U^{(t)},\U^*) =  \|\P \U^+\|_F
\]
Thus, $\deltatold = \SE_F(\U, \Ustar)$, $\deltat = \SE_F(\U^+, \Ustar) = \|\P \U^+\|_F $, and $\P \Xstar = \bm{0}$.

Let $\mu_u:= 20\kappa^2 \mu$. All our proofs use $c,C$ to denote different numerical constants in each use. The numerical values in our intermediate steps are often loose bounds to make the analysis simpler.


Recall that $\U_k:=\S_k \U$ where $\S_k = \mathrm{diag}(\xi_\jk, j \in [n])$. Similarly, let $\Ustar_k: = \S_k \Ustar$.
Thus
\[
\S_k^\intercal \S_k  = \S_k \S_k = \S_k,
\]
\[
\U_k = \S_k \U  = \sum_{j=1}^n \xi_\jk \e_j \u^j{}^\intercal, \ \Ustar_k  = \sum_{j=1}^n \xi_\jk \e_j \ustar^j{}^\intercal,
\]
and consequently,
\[
\U_k^\intercal \Ustar_k =  \U^\intercal \S_k \Ustar =  \sum_{j=1}^n \xi_\jk \u^j \ustar^j{}^\intercal
\]
Using this,
\[
\b_k = \U_k^\dag \y_k = \U_k^\dag \S_k \xstar_k = (\U_k^\intercal \U_k)^{-1} \U_k^\intercal \Ustar_k \bstar_k   = (\U^\intercal \S_k \U)^{-1} \U^\intercal \S_k \Ustar \bstar_k
\]
Thus, $\b_k - \g_k = \b_k - \U^\intercal \Ustar \bstar_k$ simplifies to
\begin{equation}
\b_k - \g_k
= (\U^\intercal \S_k \U)^{-1} [ \U^\intercal \S_k \Ustar \bstar_k -  (\U^\intercal \S_k \U) \U^\intercal \Ustar \bstar_k ]
= ( \ \underbrace{\U^\intercal \S_k \U}_{\F_k} \ )^{-1}  \underbrace{\U^\intercal \S_k  [\I - \U \U^\intercal ] \Ustar \bstar_k}_{\d_k} 
\label{eq:LS_Updt}
\end{equation}
with
\[
\d_k := \U^\intercal \S_k  [\I - \U \U^\intercal ] \Ustar \bstar_k, \   \F_k:= \U^\intercal \S_k \U
\]
Thus, letting $\F : = \mathrm{blockdiag}(\F_k, k \in [q])$,
\begin{equation*}
vec(\B- \G) = \F^{-1} vec(\D), \text{ where } \D:= [\d_1, \d_2, \dots \d_q] = \sum_k \sum_j \deltajk \u^j {[[\I - \U \U^\intercal ] \Ustar]^j}^\intercal \bstar_k  \e_k^\intercal
\end{equation*}
In the above $\F \in \Re^{rq \times rq}$ and $\D \in \Re^{r \times q}$.
Observe that
\[
\E[\S_k] = p \I, \ \E[\D] = \bm{0}
\]
Finally,
\begin{equation}
\gradU := \nabla_{\U}f(\U,\B) = 2 (\X_{\Omega} - \Y) \B^\intercal = 2 \sum_{k=1}^q \S_k (\x_k - \xstar_k) \b_k^\intercal =  2 \sum_{k=1}^q \sum_{j=1}^n \xi_\jk \e_j (\x_\jk - \xstar_\jk) \b_k^\intercal \label{eq:gradU_Updt}
\end{equation}
where $\X = \U \B$, $\Xstar = \Ustar \Bstar$,  $\x_\jk = \u^j{}^\intercal \b_k$ and  $\xstar_\jk = \ustar^j{}^\intercal \bstar_k$.
Also,
\[
\gradU^j := \e_j^\intercal \nabla_{\U}f(\U,\B)  = 2 \sum_{k=1}^q  \xi_\jk (\x_\jk - \xstar_\jk) \b_k^\intercal,
\]
Observe that
\[
\E[\gradU] = p (\X- \Xstar) \B^\intercal \text{ and } \E[\gradU^j] = p \e_j^\intercal (\X- \Xstar) \B^\intercal
\]

\subsection{Lemmas for proving Theorem \ref{thrm}} \label{lemmas}
All lemmas below assume Assumption \ref{incoh} (singular vectors' incoherence) holds. Also, everywhere $\eps < 1$.

\begin{lemma}[Initialization]
\label{lem:init}
Pick a $\delta \leq 0.2$.
Assume  $p \geq C{ \kappa^2 r^2 \mu  \log q}/{(n \delta^2 )}$. Then, w.p. at least $1 - n^{-10}$.
\begin{enumerate}
\item $\SE_F(\U^{(0)},\U^*) \leq \delta$.
\item $\U^{(0)}$ is $1.5\mu$  row-incoherent, i.e., $\| \u^{j}{}^{(0)} \| \leq 1.5 \mu \sqrt{r/n}$ for all $j \in [n]$.
\end{enumerate}
\end{lemma}

\begin{lemma}[LS step analysis: error bound for $\B$]
\label{lem:BerrLmma}
Assume $\| {\u^j} \| \leq  \mu_u \sqrt{r/n}$. Then, 
w.p. at least $1 -  \exp(\log q - c \frac{\eps^2 p n}{r^2 \mu_u^2} )$, 
\[
\| \B - \G  \|_F \leq  \eps \deltatold\sigmax
\]
\end{lemma}

\begin{lemma}[Implications of error bound for $\B$]
\label{lem:Berr_implic}
Assume $\| \B - \G  \|_F \leq \deltatold \sigmax$. Then,
\ben
\item $\| \X - \Xstar  \|_F \leq 2 \deltatold\sigmax$.
\item $\sigma_{\max}(\B) \leq (1 + \deltatold)\sigmax$ and $\sigma_{\min}(\B) \geq  \sqrt{1 - \deltatold{}^2}\sigma_{\min}^*  - \deltatold\sigmax$.
\item Thus, if $\deltatold \le c/\kappa$, then $\sigma_{\min}(\B) \ge 0.9\sigmin$ and $\sigma_{\max}(\B) \le 1.1\sigmax$.
\een

\end{lemma}

\begin{lemma}[Incoherence of $\B$]
\label{lem:BIncohLmma}
Assume that $\| {\u^j} \| \leq  \mu_u \sqrt{r/n}$. Then, w.p. greater than $1 - \exp(\log  q - c \frac{n p}{\mu_u^2 r})$, 
$$\| \b_k \| \leq 1.1\sigmax \mu \sqrt{{r}/{q}}  \ \forall k \in [q].$$
\end{lemma}

\begin{lemma}[Gradient expression and bounds]
\label{lem:Ugrad}
Assume that $\| \B - \G  \|_F \leq \deltatold\sigmax$ with $\deltatold \le c/\kappa$, $\| {\u^j} \| \leq  \mu_u \sqrt{r/n}$, $\| \b_k \| \leq \sigmax\mu \sqrt{r/q}$.
Then,
\ben
\item  $\|\E[\gradU]\|_F \leq 2.5 p \deltatold \sigmaxTwo$.	


\item $\| \gradU - \E[\gradU]  \|_F \leq \eps p   \deltatold \sigminTwo$, w.p. at least $1 - \exp(\log q - c \frac{\eps^2 p n}{\max(\kappa^4 \mu^2, \kappa^2 \mu_u\mu) r^2} ) $
\een
\end{lemma}

\begin{lemma}[Incoherence of $\U$]
\label{lem:rowIncoh}
Assume that  
$\| \b_k \| \leq 1.1 \sigmax \mu \sqrt{r/q}$.
Then, w.p. greater than $1 - \exp(\log q - c \eps^2 p n / \kappa^4 \mu^2 r^2 )$
\ben
\item $\| \gradUj - \E[\gradUj] \| \le \eps p  \max(\|{\u^j}\|,\|{\ustar^j}\|)  \sigminTwo$
\item Further, if  $\| \B - \G  \|_F \leq \deltatold\sigmax$ with $\deltatold < c/\kappa$, and
GD step size $\eta \le 0.5/(p \sigmax^2)$, then
\[
\|{\u^j}^{+}\|	\le  (1 - 0.15/\kappa^2)\|{\u^j}\|  +  2 \|{\ustar{}^j}\|.
\]
\een
\end{lemma}

\begin{proof}
All these lemmas are proved in Sec \ref{lemma_proofs}. 
\end{proof}

\subsection{Proof of Theorem \ref{thrm}}\label{thrm_proof}

We prove the following result, Claim \ref{mainclaim}, using an induction argument and the lemmas from Sec. \ref{lemmas}.
Theorem \ref{thrm} follows from the first claim of this result along with using the following:
\bi
\item To ensure $\delta^{(T)} \le \eps$, we need $(1-c/\kappa^2)^T \cdot (c/ \kappa^2) \le \eps$ or that $T = C \kappa^2 \log (1/\eps)$.
\item Since we use sample splitting, the total number of samples needed for all $T$ iterations is $2T \cdot n p \ge C \kappa^4 \mu^2 r^2 \log q  \cdot \kappa^2 \log (1/\eps) $. If all samples are generated first and then subsets used across different iterations (as specified in our stepwise algorithm, Algorithm \ref{altgdmin}), we need to assume that an entry is observed w.p. $p_{tot} = 2T \cdot p$ and $n p_{tot} \ge C \kappa^4 \mu^2 r^2 \log q  \cdot \kappa^2 \log (1/\eps)$.
\ei

\begin{claim}\label{mainclaim}
Assume everything stated in Theorem \ref{thrm} and $n p \ge C \kappa^4 \mu^2 r^2 \log q $.
For all times $\tau \ge 0$, the following hold w.p. at least $1 - \tau/ n^3$:
(i) $\delta^{(\tau)} \le (1 - c/ \kappa^2)^\tau \cdot (c/\kappa^2)$; and
(ii) $\|{\u^j}^{(\tau)}\| \le  (1 - \frac{0.15}{\kappa^2})^{\tau} 1.5 \mu \sqrt{r/n}  + [1 + (1 - \frac{0.15}{\kappa^2})  + \dots + (1 - \frac{0.15}{\kappa^2})^{\tau-1}]  2 \|{\ustar{}^j}\|$
\end{claim}

{\em Base case:}
Lemma \ref{lem:init} shows that $\deltazero \le \delta = 0.1/ \kappa^2$ and $\|{\u^j}^{(0)}\| \le 1.5 \mu \sqrt{r/n}$.
This proves (i) and (ii) for $\tau=0$. 

{\em Induction assumption:} Assume that the claim holds for $\tau=t-1$.

{\em Induction step:} Consider $\tau=t$.
Recall from Algorithm that $\tilde\U^+ = \U - \eta \gradU$, and $\U^+ =\tilde\U^+ \R^+{}^{-1}$ where $\tilde\U^+ \qreq \U^+ \R^+$. Recall from Sec. \ref{prelims} that $\P \Xstar = \bm{0}$, $\E[\gradU] =  p (\U\B - \Xstar)\B^\intercal$,   $\|\P\U\|_F = \deltatold$. Also, using Weyl's inequality, $\sigmamin(\R^+) = \sigmamin(\tilde\U^+) \ge \sigmamin(\U) - \eta ||\gradU|| = 1 -  \eta ||\gradU||$.
Using these,
\begin{align}
\deltat & =   \SE_F(\U^+, \Ustar)
	 = 	 \| \P \U^+ \|_F \notag\\
	& \le     \|\P \tilde\U^+\|_F \cdot \|(\R^+)^{-1}\| =  \|\P \tilde\U^+\|_F/\sigma_{\min}(\tilde\U^+)  \notag\\
	& \le   \frac{\|\P (\U - \eta \E[\gradU] + \eta \E[\gradU] - \eta \gradU)\|_F}{(1 - \eta \|\gradU\|)} \notag \\
	& \le  \frac{\|\P (\U - \eta p(\U \B - \Xstar)\B^\intercal)\|_F + \eta \|\E[\gradU] - \gradU)\|_F}{(1 - \eta \|\E[\gradU]\| - \eta \| \E[\gradU] - \gradU\|)}\notag \\
	& \le  \frac{\deltatold \cdot \|\I - \eta p \B \B^\intercal\| + \eta \|\E[\gradU] - \gradU)\|_F}{(1 - \eta \|\E[\gradU]\| - \eta \| \E[\gradU] - \gradU\|)}
	\label{eq:SD_interm}
\end{align}

By the induction assumption, $\deltatold \le (1 - c/ \kappa^2)^{t-1} \cdot (c/\kappa^2) \le (c/\kappa^2)$ and
\begin{align}
\|\u^j\| = \|{\u^j}^{(t-1)}\| & \le (1 - \frac{0.15}{\kappa^2})^{t-1} \|{\u^j}^{(0)}\| + [1 + (1 - \frac{0.15}{\kappa^2})  + \dots + (1 - \frac{0.15}{\kappa^2})^{t-2}]  2 \|{\ustar{}^j}\| \notag \\
& \le  \|{\u^j}^{(0)}\| + \frac{\kappa^2}{0.15}  2 \|{\ustar{}^j}\| \le (1.5\mu + 14\kappa^2 \mu)\sqrt{r/n} \le  \mu_u \sqrt{r/n}
\label{ujincoh}
\end{align}
The last inequality above used the infinite geometric series bound. This shows that $\U$ is $\mu_u$-row-incoherent.

Using \eqref{ujincoh} and $\deltatold \le (c/\kappa^2)$ , Lemmas \ref{lem:BerrLmma} and \ref{lem:BIncohLmma} hold, i.e. the bound on $\|\B-\G\|_F$ and $\|\b_k\|$ holds.
This then implies that all claims of Lemma \ref{lem:Berr_implic} hold too.
This then implies that Lemmas \ref{lem:Ugrad} and \ref{lem:rowIncoh} hold.
All the lemmas hold with probability at least $1 - 1/n^3$ if $p$ is such that all the probabilities stated in all the lemmas are lower bounded by $1 - 0.1/n^3$. Using $\mu_u = 8 \kappa^2 \mu$, clearly, this is true if
\[
n p \ge C \kappa^4 \mu^2 r^2 \log q
\]

By Lemmas \ref{lem:BerrLmma} and \ref{lem:Berr_implic}, $\sigma_{\min}(\B) \ge 0.9\sigmin$ and $\|\B\|\le 1.1 \sigmax$. Using these, if $\eta \leq 0.5/(p\sigmaxTwo)$, then $\I - \eta p \B\B^\intercal$ is positive semi-definite (psd) and hence $\| \I - \eta p \B\B^\intercal \| = \lambda_{\max}(\I - \eta p \B\B^\intercal) \leq 1 - 0.8 \eta p \sigminTwo$. 
Using Lemma \ref{lem:BIncohLmma}, $\B$ is $1.1\mu$-column-incoherent.
Using Lemma \ref{lem:Ugrad},  $\|\E[\gradU]\| \leq 2.5 p \deltatold \sigmaxTwo$ and $\| \gradU - \E[\gradU]  \|_F \leq \eps p   \deltatold \sigminTwo$.  
%
Substituting these into \eqref{eq:SD_interm} with $\eps=0.01$,
\begin{align*} %
\deltat
& \leq  \frac{\deltatold\big(1 - 0.8\eta p\sigminTwo +  0.01\eta p \sigminTwo \big)}{ 1 - \deltatold\big( 2.55 \eta p  \sigmaxTwo \big)}.\\
 &\leq \deltatold (1 - 0.79\eta p\sigminTwo) (1 + 2 \deltatold( 2.55 \eta p  \sigmaxTwo )) \\
&\leq \deltatold  (1 -  (0.79 - \deltatold  5.1 \kappa^2 ) \eta p\sigminTwo  )
\leq \deltatold  (1 -   \eta p\sigminTwo ( 0.79 - 0.051 )) \le \deltatold  (1 -   0.3\eta p\sigminTwo)
\end{align*}
In the above, for the denominator term we used $1/(1-x) \le 1+2x$ for $x < 0.5$ and $\deltatold \le 0.01/\kappa^2$.
%
Setting $\eta = 0.5/(p\sigmaxTwo)$ in the final expression given above, we can conclude that
\[
\deltat \le (1 - 0.15/ \kappa^2) \deltatold \le (1 - 0.15/ \kappa^2)^t \cdot (0.1/\kappa^2).
\]
Thus claim (i) holds for $\tau=t$. Next we prove claim (ii) for $\tau=t$.
By Lemma \ref{lem:rowIncoh} and the induction assumption,
\begin{align*}
\|{\u^j}^{(t)}\| & = \|{\u^j}^{+}\| 	\le  (1 - 0.15/\kappa^2)   \|{\u^j}\| +  2 \|{\ustar{}^j}\| \\
& \le (1 - 0.15/\kappa^2)^t 1.5 \mu \sqrt{r/n}   + (1 - 0.15/\kappa^2) [1 + (1 - \frac{0.15}{\kappa^2})  + \dots + (1 - \frac{0.15}{\kappa^2})^{t-2}]  2 \|{\ustar{}^j}\|   +  2 \|{\ustar{}^j}\| \\
& = (1 - 0.15/\kappa^2)^t 1.5 \mu \sqrt{r/n}   + [1 + (1 - \frac{0.15}{\kappa^2})  + \dots + (1 - \frac{0.15}{\kappa^2})^{t-1}]  2 \|{\ustar{}^j}\|
\end{align*}


\section{Proofs of the Lemmas}\label{lemma_proofs}

\subsection{Proof of Lemma \ref{lem:init}}\label{proof_init}
By Lemmas 2.5 and 2.6 of \cite{Chen_2021}, for two $n \times r$ matrices with orthonormal columns, $\U_1, \U_2$,
\begin{align}
\SE_F(\U_1, \U_2) \le \min_{\Q \ unitary} \|\U_1 - \U_2 \Q \|_F  \le \sqrt{2} \SE_F(\U_1, \U_2) 
\label{equiv}
\end{align}
and a similar bound holds for $\SE_2$ as well. We use this fact frequently below.

Recall from the Algorithm that $\Y \svdeq \U^{(00)} \Sigma^{(00)} \V^{(00)}$, and $\Pi_{\calU}(\U^{(00)}) \qreq \U^{(0)} \R^{(0)}$.   
%
%
%
%
By Theorem 3.22 of \cite{Chen_2021} along with using \eqref{equiv} and $\SE_2(\U_1, \U_2) \le \sqrt{r} \SE_F(\U_1, \U_2)$, we have: with probability at least $1 - C q^{-10}$,
\[
\SE_F(\U^{(00)},\Ustar) \le  C \sqrt{\frac{\kappa^2 \mu r^2 \log q}{np}}
\]
Let $\Q_{*,00} = \arg\min_{\Q \ unitary} \|\U^{(00)} - \Ustar \Q \|_F$. By \eqref{equiv}, 
\[
\|\U^{(00)} - \Ustar \Q_{*,00}\|_F \le \sqrt{2} \SE_F(\U^{(00)},\Ustar) 
\]
Next we use the above two bounds and the fact that $\calU$ is a convex set\footnote{To see that $\calU$ is a convex set, let matrices $\M_1,\M_2 \in \calU$.  Let $\m_1^j$ and $ \m_2^j$ denote the rows of $\M_1$ and $\M_2$, respectively. For any $0 < \theta < 1$, $\| \theta \m_1^j + (1 - \theta)\m_2^j \| \leq \theta\| \m_1^j\| + (1 - \theta) \| \m_2^j\| \leq \theta \mu \sqrt{r/n} + (1 -\theta )\mu \sqrt{r/n}  = \mu \sqrt{r/n}.$
}
to bound $\|\Pi_{\calU}(\U^{(00)}) - \Ustar \Q_{*,00} \|_F$ as follows.
\begin{align*}
\| \Pi_{\calU}(\U^{(00)}) - \Ustar \Q_{*,00} \|_F = \| \Pi_{\calU}(\U^{(00)}) - \Pi_{\calU}(\Ustar \Q_{*,00}) \|_F
 \leq \| \U^{(00)} - \Ustar \Q_{*,00} \|_F
 \leq  \sqrt{2} \SE_F(\U^{(00)}, \Ustar) \le  C  \sqrt{\frac{\kappa^2 \mu r^2 \log q}{np}}
\end{align*}
with probability at least $1 - C q^{-10}$.
The first equality above uses the fact that, for any unitary $\Q$, $\Ustar \Q$ belongs to $\calU$. This is a direct consequence of our incoherence assumption Assumption \ref{incoh}.
The second one uses the facts that projection onto $\calU$, which is a convex set, is non-expansive \cite[eq. (9),(10)]{nashed1968decomposition}, \cite[eq. (1.5)]{zarantonello1971projections}). The projection $\Pi_{\calU}(\M) := \min_{\U \in \calU} \|\U - \M\|_F$ is non-expansive means that $\|\Pi_{\calU}(\M) - \Pi_{\calU}(\M_2)\|_F \le \|\M - \M_2\|_F$.

Finally, we use the above to bound $\SE_F(\U^{(0)}, \Ustar) = \|\P \U^{(0)}\|_F = \|\P \Pi_{\calU}(\U^{(00)}) \R^{(0)}{}^{-1} \|_F$. Using $ \P  \Ustar \Q_{*,00} = \bm{0}$, and $x/(1-x) < 1.5 x$ for any $0 \le x \le 0.3$,
we have: if $C  \sqrt{\frac{\kappa^2 \mu r^2 \log q}{np}} < 0.3$, then
\begin{align*}
\SE_F(\U^{(0)}, \Ustar) = \|\P \U^{(0)}\|_F
& \le \|\P \Pi_{\calU}(\U^{(00)})\|_F \|\R^{(0)}{}^{-1}\|  \\
& \le \frac{\| \P  \Ustar \Q_{*,00}\|_F + \|\P (\Pi_{\calU}(\U^{(00)}) - \Ustar \Q_{*,00})\|_F }{\sigma_{\min}(\Pi_{\calU}(\U^{(00)})} \\
& \le \frac{\|\Pi_{\calU}(\U^{(00)}) - \Ustar \Q_{*,00}\|_F }{1 - \|\Pi_{\calU}(\U^{(00)}) - \Ustar \Q_{*,00} \|_F}
 \le \frac{C  \sqrt{\frac{\kappa^2 \mu r^2 \log q}{np}}}{1 - C \sqrt{\frac{\kappa^2 \mu r^2 \log q}{np}}} \le 1.5 C  \sqrt{\frac{\kappa^2 \mu r^2 \log q}{np}}
\end{align*}
The denominator bound follows by Weyl's inequality, $\sigma_{\min}(\Ustar \Q_{*,00}) =  1$,  and $x/(1-x) < 1.5 x$ for $0 < x < 0.3$.
%
Thus, for a $\delta < 0.3$, w.p. at least $1 - q^{-10}$, $\SE_F(\U^{(0)}, \Ustar)  \le  \delta$ if $p \ge C \kappa^2 \mu r^2 \log q / n \delta^2$.

\subsubsection{Proof of $\mu$-incoherence} 
We have
$\u^j{}^{(0)} =  (\Pi_{\calU}(\U^{(00)}))^j \R^{(0)}{}^{-1}$. Thus, using the denominator bound from above,
\[
\| \u^j{}^{(0)} \| \leq \| (\Pi_{\calU}(\U^{(00)}))^j \|  \|\R^{(0)}{}^{-1} \|
= \frac{\| (\Pi_{\calU}(\U^{(00)}))^j \|}{\sigma_{\min}( \Pi_{\calU}(\U^{(00)})}
\le \frac{\mu\sqrt{r/n}}{1 - C \sqrt{\frac{\kappa^2 \mu r^2 \log q}{np}}} 
\]
Our assumed bound on $p$ implies that the denominator is at least $1-\delta$ for a $\delta < 0.3$. This then implies the above is upper bounded by $1.5 \mu\sqrt{r/n}$.

\subsection{Proof of Lemma \ref{lem:BerrLmma}}\label{proof_Blemma}
Let $\delta = \deltatold$.
Recall from Sec. \ref{prelims} that 
$vec(\B-G) = \F^{-1} vec(\D)$ with $\F,\D$ defined there. Here $\F  \in \Re^{rq \times rq}$ and $\D \in \Re^{r \times q}$.
Thus,
\begin{equation}
\|\B- \G\|_F =  \|vec(\B - \G)\| =\| \F^{-1}vec(\D) \| \leq \|\F^{-1}\| \cdot \|vec(\D)\| =  \|\F^{-1}\| \cdot \|\D\|_F \le \|\F^{-1}\| \cdot \sqrt{r} \|\D\|, \label{eq:B-GExprsn}
\end{equation}

By the matrix-Bernstein inequality, 
$\|\U_k^\intercal \U_k - p \I\| \le \eps p, \,\,\, \text{w.p. greater than {$1 - \exp(\log 2r - c \frac{\eps^2 pn}{\mu_u^2 r})$.}}
$
Applying the union bound over all $q$ diagonal blocks, and using the fact that $\sigma_{\min}(\F) = \min_{k} \sigma_{\min}(\U_k^\intercal \U_k))$ (since $\F$ is block diagonal),    
\begin{equation}
\| \F^{-1} \| \leq \frac{1}{(1 - \eps)p}, \,\,\, \text{ w.p. greater than $1 - \exp(\log q + \log 2r - c \frac{ \eps^2 pn}{\mu_u^2 r})$.}
\end{equation}
(the above also follows by Lemma C.7 of \cite{lowrank_altmin}).

Using the above bound and the bound on $\| \D \|$ that we derive below 
\begin{equation}
\| \B - \G \|_F \leq \sqrt{r} \|\F^{-1}\|  \|\D\|
\le
2 \eps  \sqrt{r} \delta  \sigmax, \,\,\,\text{w.p. greater than  $1 -  \exp(\log q -  \frac{\eps^2 p n}{r\mu_u^2} )$}.
\label{eq:B-G_EpsBnd0}
\end{equation}
Let $\eps = \eps_1/2\sqrt{r}$. Then,
\begin{equation}
\| \B - \G \|_F \leq \|\F^{-1}\| \cdot \sqrt{r} \|\D\|
\le
\eps_1 \delta  \sigmax, \,\,\,\text{w.p. greater than  $1 -  \exp(\log q -  \frac{\eps^2 p n}{r^2 \mu_u^2} )$}.
\label{eq:B-G_EpsBnd}
\end{equation}

\subsubsection{$\| D \|$ bound}
Let
\[
\M: =  [(\I - \U \U^\intercal) \Ustar]
\]
Recall from Sec. \ref{prelims} that
\begin{equation}
\D = \sum_k \d_k \e_k^\intercal = \U^\intercal \S_k \M \bstar_k \e_k^\intercal = \sum_k \sum_j \underbrace{\deltajk \u^j  \M^j{}^\intercal \bstar_k \e_k^\intercal}_{\Z_\jk}
\text{ and } \E[\D] = \bm{0}
\label{eq:Dexprsn}
\end{equation}
We bound $\|\D\| = \|\D - \E[\D]\|$ using matrix Bernstein inequality. Using $n \le q$,
\begin{align*}
L =  \max_\jk \| \deltajk \u_j \M^j{}^\intercal \bstar_k \e_k^\intercal \| \le \max_\jk \| \M^j\| \|\u_j\|  \|\bstar_k\| \leq \delta \sigmax \mu \mu_u {r}/{n}
\end{align*}
where we used $\| \M^j\| \le \|\M\| \le \delta$.
%
Also, using $\E[\deltajk] = \E[\deltajk^2] = p$,
\begin{align*}
\sigma_1^2
& = \| \E[\sum_\jk \Z_\jk^\intercal  \Z_\jk] \|
=  p \| \sum_\jk  (\M^j{}^\intercal \bstar_k)^2 \|\u_j\|^2 \e_k  \e_k^\intercal \| \\
& \leq p \mu_u^2 \frac{r}{n}  \| \M\B^*\|_F^2 \leq 2p \mu_u^2 \frac{r}{n} \| \M \|_{F}^2 \| \B^* \|^2 \leq 2p \mu_u^2 \frac{r}{n} \delta^2  \sigmaxTwo.
\end{align*}
and, proceeding in a very similar fashion,
\begin{align*}
\sigma_2^2
& = \| \E[\sum_\jk\Z_\jk\Z_\jk^\intercal ]\|
=  p \| \sum_\jk  (\M^j{}^\intercal \bstar_k)^2  \u_j  \u_j^\intercal \| \\
&  \leq p \mu_u^2 \frac{r}{n}  \| \M\B^*\|_F^2 \leq 2p \mu_u^2 \frac{r}{n} \| \M \|_{F}^2 \| \B^* \|^2 \leq 2p \mu_u^2 \frac{r}{n} \delta^2  \sigmaxTwo.   \label{eq:sigma2}
\end{align*}
Thus,  $\sigma^2 = \max(\sigma_1^2,\sigma_2^2) = \sigma_2^2$, {for $r < q$}. Setting $t = \eps p \text{{$\delta$}} \sigmax$, we have
\begin{equation}
\frac{t}{L} = \frac{\eps p n}{r\mu  \mu_u}, \,\,\, \frac{t^2}{\sigma^2} = \frac{\eps^2 p n}{2r\mu_u^2}. \label{eq:muStep}
\end{equation}
Thus, by the Matrix-Bernstein inequality, and using $\mu \le \mu_u $,
\begin{equation}
\| \D\| = \| \D - \E[\D] \| \leq  \eps p \text{{$\delta$}} \sigmax, \text{w.p. greater than } 1 -  \exp(\log q -  c \frac{\eps^2 p n}{r\mu_u^2} ). \label{eq:DopBnd}
\end{equation}



\subsection{Proof of Lemma \ref{lem:Berr_implic}}\label{proof_Blemma_implic}
Let $\delta = \deltatold$.
Writing $\Xstar = \U\G + (\I - \U\U^\intercal)\Xstar$, and $\X = \U \B$, we have
$
\| \Xstar - \X \|_F 
\leq
\| \B - \G \|_F  + \| (\I - \U\U^\intercal)\U^*\B^* \|_F\leq
\| \B - \G \|_F + \| (\I - \U\U^\intercal)\U^*\|_F  \| \B^* \| \leq
c \delta \sigmax + \delta \sigmax.
$
For ii), using the bound on $\| \B - \G \|_F$, $\| \B \| = \| \B - \G   + \G \| \leq \| \B - \G \| + \| \G \| \leq  \delta \sigmax +\sigmax.$  For iii), $\sigma_{\min}(\B)\geq \sigma_{\min}(\G) - \sigma_{\max}(\B -\G) \geq \sqrt{1 - \delta^2} \sigmin -\delta \sqrt r\sigmax$. Here we used $\sigma_{\min}(\G) = \sigma_{\min}(\U^\intercal \Ustar\Bstar) \geq \sigma_{\min}(\U^\intercal \U^*) \sigmin \geq \sqrt{1 -\delta^2} \sigmin.$


\subsection{Proof of Lemma \ref{lem:BIncohLmma}}
\label{subsec:bkIncoh}
Since $\b_k =  (\U_k^\intercal\U_k)^{-1}  \U_k^\intercal \Ustar_k  \bstar_k$,
\[
\|\b_k\| \le \| (\U_k^\intercal\U_k)^{-1}\| \cdot \|\U_k^\intercal \Ustar_k\| \cdot \|\bstar_k\|
\]
By Lemma C.6 of  \cite{lowrank_altmin}:
\begin{equation}
\| (\U_k^\intercal \U_k)^{-1} \| \leq  \frac{1}{(1 - \epsilon)p} \ \text{ w.p. at least } 1 - \exp(\log r - \eps^2 p n / \mu_u^2 r )
\label{eq:opBndTrm1Fnl}
\end{equation}
To bound  $\| \U_k^\intercal \Ustar_k\|$, first note that 
\[
\E[\U_k^\intercal \Ustar_k ] = p \U^\intercal \Ustar.
\]
Next we bound  $\| \U_k^\intercal \Ustar_k  - p \U^\intercal \Ustar \|$ using matrix Bernstein inequality. Recall the expression for $\U_k^\intercal \Ustar_k$ from Sec. \ref{prelims}.
Let $\Z_j = (\deltajk - p){\u^j} {\ustar^j}{}^{\intercal}$.
As done in earlier proofs, we can show that
$$ L = \max_j \|\Z_j\| \le \mu \mu_u \sqrt{r/n}, \ \text{ and}$$
$$\sigma^2 = \max(\| \sum_j \E [\Z_j \Z_j^\intercal] \|, \| \sum_j \E[\Z_j^\intercal \Z_j]\| )
\le 2p \mu_u^2 \frac{r}{n}.$$
For this, we used the fact that $\E[(\deltajk-p)^2] = 2p(1-p) \le 2p$,
$\|\sum_j \| \ustar_j \|^2 \u_j \u_j^\intercal\| \le \mu^2 (r/n) \|\U\| = \mu^2 (r/n) $ and $\|\sum_j \| \u^j \|^2 \ustar_j \ustar_j{}^\intercal\| \le \mu_u^2 (r/n) \|\Ustar\| = \mu_u^2 (r/n) $ and $\mu \le \mu_u$.

Thus, by matrix Bernstein, $\| \U_k^\intercal \Ustar_k - p \U^\intercal \U^* \| \le \eps p $ w.p. at least $1 - \exp(\log r - \eps^2 p n / \mu_u^2 r ) $. Hence, with this probability,
\begin{equation}
\| \U_k^\intercal \Ustar_k \| \leq \| \U_k^\intercal \Ustar_k - \U^\intercal \Ustar \| + p \|\U^\intercal \Ustar\| \le \eps p + p = (1+\eps) p.
\end{equation}
Thus, letting $\eps = 0.1$,  w.p. at least $1 - \exp(\log r - c p n / \mu_u^2 r ) $,
$
\|\b_k\| \le 1.1 \|\bstar_k\| \le 1.1\mu \sqrt{r/q} \sigmax.
$
Here we obtained a bound on $||\b_k||$ for a given $k$. By union bound, the above bound holds for all $k \in [q]$  w.p. at least $1 - q \exp(\log r - c p n / \mu_u^2 r ) = 1 - \exp(\log q + \log r - c p n / \mu_u^2 r )$.

\subsection{Proof of Lemma \ref{lem:Ugrad} }
\label{subsec:gradUBnd}
Let $\delta = \deltatold$.
Recall that $\gradU = \sum_\jk  \deltajk \e_j (\x_\jk - \x^*_\jk)\b_k^\intercal.$
Using the lemma assumption, $\|\B - \G\|_F \le  \delta \sigmax$. Using this and Lemma \ref{lem:Berr_implic}, $\|\X - \Xstar\|_F \le  2\delta \sigmax$, $\|\B\| \le 1.1 \sigmax$,
We use these bounds in the proof below.

Observe that
\[
\E[\gradU] = p (\X- \Xstar) \B^\intercal,
\]
Thus, using the assumed $\B-\G$ bound,
\[
\|\E[\gradU] \|_F \leq p \| \X - \Xstar \|_F \| \B \|  \leq 2.2 p \delta \sigmaxTwo
\]

Next we bound the deviation using matrix Bernstein inequality.
Writing $\X = \U \B$, $\Xstar = \U\G + (\I - \U \U^\intercal)\Xstar$, and using $n \le q$ 
\begin{align*}
	\max_\jk |\x_\jk - \x^*_\jk|
	& = |\e_j^\intercal(\X - \Xstar)\e_k|  \\
	& \le  \|\e_j^\intercal\U \| \|\B - \G\| + \|(\I - \U \U^\intercal)\Ustar\| \|\Bstar \e_k\| \\
	& \le \mu_u \sqrt{r/n}  \delta \sigmax + \mu \sqrt{r/q} \delta  \sigmax
	\le 2 \mu_u \sqrt{r/n} \delta \sigmax
\end{align*}
Using this,
\[
L = \max_\jk \|(\x_\jk - \xstar_\jk) \b_k^\intercal\| \le    \max_\jk \|(\x_\jk - \xstar_\jk)\| \max_k \| \b_k\| \le
2 \mu_u \sqrt{r/n} \delta \sigmax \cdot \mu \sqrt{r/q} \sigmax \le 2\mu_u \mu (r/n) \delta \sigmaxTwo,
\]
To bound $\sigma_1^2 = \|\sum_\jk\E[\Z_\jk\Z_\jk^\intercal]\|$ and $\sigma_2^2 = \|\sum_\jk\E[\Z_\jk^\intercal \Z_\jk]\|$, where $\Z_\jk = \deltajk \e_j (\x_\jk - \x^*_\jk)\b_k^\intercal$, using $\E[(\deltajk-p)^2] = 2p(1-p) \le 2p$
we have
\begin{align}
	&\sigma_1^2
	= 2p \|\sum_\jk (\x_\jk - \xstar_\jk)^2 \e_j \b_k^\intercal \b_k \e_j^\intercal\| 
	\le 2 p \|\b_k\|^2 \|\X - \Xstar\|_F^2
	\le 2p\mu^2(r/q) \sigmaxTwo \cdot  ( \delta \sigmax )^2
	= 2p\mu^2(r/q) \delta^2 \sigma^{*4}_{\max}. \notag\\
	&\sigma_2^2
	=2 p \| \sum_\jk (\x_\jk - \xstar_\jk)^2 \e_j^\intercal \e_j \b_k\b_k^\intercal \|
	\le 2 p \|\b_k\|^2 \|\X - \Xstar\|_F^2  = \sigma_1^2. \notag
\end{align}
Setting $t = \eps p  \delta \sigminTwo $, we have
\[
\frac{t^2}{\sigma^2} = c\frac{\eps^2 p^2  \delta^2 \sigmin{}^4}{p \mu^2(r/q) \delta^2 \sigmaxFour}  = \frac{\eps^2 p q}{\kappa^4 \mu^2 r}, \ \  
\frac{t}{L} = c\frac{ \eps p   \delta \sigminTwo}{\mu_u\mu (r / n) \delta \sigmaxTwo} =  \frac{ \eps p n }{\kappa^2 \mu_u\mu r}.
\]
Thus,
\[
\min\left(\frac{t^2}{\sigma^2} , \frac{t}{L} \right) = c\frac{\eps^2 p n}{\max(\kappa^4 \mu^2, \kappa^2 \mu_u\mu) r}
\]
and so, by matrix Bernstein, w.p. at least $1 - \exp(\log q - c\frac{\eps^2 p n}{\max(\kappa^4 \mu^2, \kappa^2 \mu_u\mu) r} ) $,
$
\|\gradU - \E[\gradU] \| \le \eps p   \delta \sigminTwo.
$

By setting $\eps = \eps_1/\sqrt{r}$,  $\|\gradU - \E[\gradU] \|_F \le \sqrt{r} \|\gradU - \E[\gradU] \| \le \eps_1  p   \delta \sigminTwo$ w.p. at least $1 - \exp(\log q - \frac{\eps_1^2 p n}{\max(\kappa^4 \mu^2, \kappa^2 \mu_u\mu) r^2} ) $.

\subsection{Proof of Lemma \ref{lem:rowIncoh}}\label{subsec:rowIncohProof}
Let $\delta = \deltatold$. This is used only for proving the second part.

Recall from Sec. \ref{prelims} that
	\[
	\gradUj = \sum_k \deltajk (\x_\jk - \xstar_\jk) \b_k^\intercal,
	\]
	To apply matrix Bernstein, we need to bound $\sigma_1^2 \equiv \|\E[\sum_k \deltajk^2 (\x_\jk - \x^*_\jk)^2\b_k^\intercal\b_k]\|$, $\sigma_2^2 \equiv \|\E[\sum_k \deltajk^2 (\x_\jk - \x^*_\jk)^2\b_k\b_k^\intercal]\|$, and $L = \max_k \|(\x_\jk - \xstar_\jk) \b_k \|$.
	Using
	\[
	|\x_\jk - \xstar_\jk| \le 2 \max(|\x_\jk|, |\xstar_\jk|) \le  2\max( \|{\u^j}\| \|\b_k\|, \|{\ustar{}^j}\| \|\bstar_k\|) \le 2 \max( \|{\u^j}\|, \|{\ustar{}^j}\|) \max(\|\b_k\|,\|\bstar_k\|).
	\]
	and  $\max(\|\bstar_k\|, \|\b_k\|) \le 1.1 \mu \sqrt{r/q} \sigmax $,
	\begin{align}
		L \le  \max_k |\x_\jk - \xstar_\jk| \max_k \|\b_k\| &\le  2\max( \|{\u^j}\|, \|{\ustar{}^j}\|) \max_k \max(\|\b_k\|,\|\bstar_k\|) \max_k \|\b_k\| \notag \\
&\le  2\max( \|{\u^j}\|, \|{\ustar{}^j}\|)  \mu^2 (r/q) \sigmaxTwo.
	\end{align}
We use the above seemingly loose bound because we do not need $\deltatold$ in this bound. Instead we need to the bound to be of the form $\max(\|\u^j\|, \|\ustar_j\|)$ times a factor of $ \sqrt{r/q} \sigmax$

	The variance $\sigma_1^2 = \|\E[\sum_k \deltajk (\x_\jk - \x^*_\jk)^2\b_k^\intercal\b_k]\|$ can be bounded using similar ideas as follows.
	\begin{align}
		\sigma_1^2
		= \| \sum_k p (\x_\jk - \xstar_\jk)^2 \b_k^\intercal \b_k \|
		& \le 2p  {\max (\| {\u^j} \|_2, \| \ustar^j \|_2)}^2 \max_k (\|\b_k\|,\|\b^*_k\|)^2 \sum_k \| \b_k \|_2^2\notag\\
		& \le 2p   \max{(\| {\u^j} \|_2, \| \ustar^j \|_2)}^2\mu^2 (r^2/q) \sigmaxFour.
	\end{align}
	Similarly, $\sigma_2^2 \leq 2p   \max{(\| {\u^j} \|_2, \| \ustar^j \|_2)}^2\mu^2 (r^2/q) \sigmaxFour.$
	By the matrix Bernstein inequality with $t = \eps p  \max( \|{\u^j}\|,  \|{\u^{*j}}\|) \sigminTwo$, and noting that $\sigma^2 = \max(\sigma_1^2,\sigma_2^2) = \sigma_1^2$, we have w.p. at least  $1 - \exp(\log q - \eps^2 p n /\mu^2 \kappa^4 r^2 )$,
	\begin{equation}
		\|\E[\gradUj] - \gradUj \| \le \eps p  \max( \|{\u^j}\|,  \|{\ustar{}^{j}}\|)  \sigminTwo.
		\label{eq:rowGradjDvtn}
	\end{equation}
	This completes the proof for the first part of the lemma.

\subsubsection{Proof of second part}
In line 7 of Algorithm 1, adding/subtracting $\E[\gradUj] = p (\u^{j}{} \B \B^\intercal -  {\ustar{}^j}  \Bstar \B^\intercal)$,
\begin{equation}
\tilde\u^j{}^+  =\u^{j}{}  - \eta \gradUj =  \u^{j}{}  ( \I - \eta p \B \B^\intercal)
+ \eta p {\ustar{}^j}  \Bstar \B^\intercal + \eta (\E[\gradUj] - \gradUj), \label{eq:ujBnd}
\end{equation}
Since we assumed $\delta \leq c/\kappa$, by Lemma \ref{lem:Berr_implic}, $\sigma_{\min}(\B) \ge 0.9 \sigmin$ and $\sigma_{\max}(\B) \le 1.1 \sigmax$. Thus, if  $\eta < 0.5 / p \sigmaxTwo$ then, $\I - \eta p \B \B^\intercal$ is positive semi-definite (psd) and so $\|\I - \eta p \B \B^\intercal\| = \lambda_{\max}(\I - \eta p \B \B^\intercal)  = 1 - \eta p \sigma_{\min}^2(\B) \leq  1 - 0.9 \eta p \sigmin^2$.

Thus, using the above bound on $\|\E[\gradUj] - \gradUj \|$,
if  $\eta < 0.5 / p \sigmaxTwo$,
\begin{align}
	\|\tilde\u^j{}^+ \|
	&\le \|{\u^{j}{}}\| (1 - 0.9\eta p \sigminTwo)  +  \eta p \|{\ustar{}^j}\| \sigmaxTwo  + \eps \eta p \sigminTwo\max( \|{\u^{j}{}}\|,  \|{\u^{*j}}\|) \notag\\
	&\le  (1 - (0.9-\eps)\eta p \sigminTwo )\max( \|{\u^{j}{}}\|,  \|{\u^{*j}}\|)   + \eta p \sigmaxTwo \|{\ustar{}^j}\|,
	\label{eq:uTildeBnd}
\end{align}
w.p. at least $1 - \exp(\log q - \eps^2 p n /\mu^2 \kappa^4 r^2 )$. We bound $\| \u^{j(t+1)} \| \leq  \| (\R^+)^{-1} \| \cdot \| \tilde{\u}^{j (t+1)}\|$, where $\tilde \U^+ \overset{\text{QR}}{=} \U^+ \R^+$ next. Using Lemma \ref{lem:Ugrad},
	\begin{equation}
		\| (\R^+)^{-1} \| =\frac{1}{\sigma_{\min}(\U - \eta  \| \gradU \| )}
		\leq  \frac{1}{1 - \eta 2.1 p \delta \sigmaxTwo}
		\leq \frac{1}{1 - 0.25\eta p  \sigmaxTwo}
		\leq 1 + 0.5 \eta p  \sigmaxTwo.  \label{eq:RinvBnd}
	\end{equation}
w.p. given in Lemma \ref{lem:Ugrad}.
Thus, using \eqref{eq:uTildeBnd} and \eqref{eq:RinvBnd} in $\| {\u^j}^+ \| \leq  \| (\R^+)^{-1} \| \cdot \| \tilde{\u}^{j (t+1)} \|$,
	\begin{align*}
		\|\u^j{}^+ \|
		& \le (1 + 0.5 \eta p  \sigmaxTwo) (1 - (0.9-\eps) \eta p \sigminTwo )\max( \|{\u^{j}{}}\| \|{\u^{*j}}\|) + (1 + 0.5 \eta p  \sigminTwo) \eta p \sigmaxTwo \|{\ustar{}^j}\|\\
		& \le (1 - (0.4-\eps) \eta p \sigminTwo )\max( \|{\u^{j}{}}\|,  \|{\u^{*j}}\|) + (1 + 0.5 \eta p  \sigmaxTwo) \eta p \sigmaxTwo \|{\ustar{}^j}\| \\
		& \le (1 - (0.4-\eps) \eta p \sigminTwo ) \max( \|{\u^{j}{}}\|,  \|{\u^{*j}}\|) + (1 + 0.25/\kappa^2) 0.5 \|{\ustar{}^j}\| \\
		& \le (1 - \frac{0.15}{\kappa^2}) \max( \|{\u^{j}{}}\|,  \|{\u^{*j}}\|) + 0.7 \|{\ustar{}^j}\|
	\end{align*}
	where the last bound follows by setting $\eta = 0.5/p \sigmaxTwo$ and setting $\eps = 0.1$ in \eqref{eq:rowGradjDvtn}. Thus, we have shown that if $\eta = 0.5/p \sigmaxTwo$, and if  $\delta \leq c/\kappa^2$, with probability exceeding $1 - 4/n^3$,
	\begin{align*}
		\|{\u^j}^+ \|  & \le (1 - \frac{0.15}{\kappa^2})\max(\|{\u^{j}{}}\|,  \|\ustar{}^{j}\|) + 0.7 \|{\ustar{}^j}\|\\
		&\leq (1 - \frac{0.15}{\kappa^2})\| \u^{j}{} \| + 2 \| \ustar{}^j\|.
	\end{align*}

\section{Guarantee for Noisy LRMC}\label{noisy_lrmc_th}
Consider the noisy LRMC problem defined as follows. We observe
\[
\Y := \Xstar_\Omega + \W_\Omega
\]
We do not make any assumption on the noise $\W$ (deterministic $\W$).  Unlike the standard Gaussian assumption, this assumption is weaker: it allows the noise entries to be anything, all entries could be positive too for example.

We can show the following by borrowing the overall proof approach for modifying a noise-free case guarantee for an iterative algorithm from past works \cite{twf,lrpr_best} (the latter used a similar approach for analyzing the AltMin algorithm for LR phase retrieval).



\begin{corollary}
Let
\[
\eps_{noise}  := \max_{jk} \frac{\lvert w_{jk} \rvert}{\lvert x^*_{jk} \rvert}
\]
In the setting of Theorem \ref{thrm}, if
$
\eps_{noise} \le  \frac{c}{\sqrt{r} \kappa^3},
$
then, by setting $T = C \kappa^2 \log(1/\eps_{noise})$, we can guarantee that
\[
\SE_F(\U^{(T)},\Ustar) \le \kappa^2 \sqrt{r} \eps_{noise}   
\]
$\ \text{ and } \  \|\X^{(T)}- \Xstar\|_F \le \SE_F(\U^{(T)},\Ustar) \sigmax.$
In general, for any $\eps$, by setting $T = C \kappa^2 \log(1/\eps)$, we can ensure that
\[
\SE_F(\U^{(T)},\Ustar) \le \max(\eps, \kappa^2 \sqrt{r} \eps_{noise})
\]
 and $\|\X- \Xstar\|_F \le \SE_F(\U^{(T)},\Ustar)   \sigmax $.   
\label{cor_noisy}
\end{corollary}

\begin{proof}
Corollary \ref{cor_noisy} extends our noise-free case proof to the noisy case using the following simple ideas.
First we assume that the noise level is small enough so that accurate initialization is possible; this requires noise level, $\eps_{noise} \sqrt{r}$,  to be of order $c \deltazero \sigmin$ for our required value of $\deltazero = c/\kappa^2$. This helps ensure that the initialization error is bounded by $\deltazero$.

Next, at each iteration, we attempt to bound terms and prove incoherence for the updated $\U,\B$, in order to guarantee error decay similar to the noise-free case. For this, we need the noise level to be such that (i) the error in recovering $\Bstar$ is of the same order as in the noise-free case; and (ii) the same is true for the bounds on the gradient norm. Both of these are ensured if  $\eps_{noise}  \kappa^2 \sqrt{r} \le c \deltatold$ for a $c < 1$, e.g., $c=0.1$. Details are given in Appendix \ref{proof_noisy}.
\end{proof}

\section{Improved Guarantees for AltMin and Smooth-AltMin} \label{altmincorolls}

Using the same lemmas used to prove the AltGDmin guarantee, we are also able to improve the guarantee for AltMin and Smooth-AltMin \cite{hardt2014understanding}, as long as both are initialized as given in Algorithm 1:  the initial estimate of $\U$ is clipped using row-wise clipping. We borrow this clipping idea from \cite{rpca_gd}.

\subsection{Improving the result for AltMin}

\begin{corollary}[Improved AltMin \cite{lowrank_altmin} Guarantee]
Consider the AltMin algorithm initialized using our initialization (lines 2-4) of Algorithm 1.
	Assume that Assumption \ref{incoh} holds and that, at each iteration $t$, entries of $\Xstar$ are observed independently with probability $p$ satisfying
	\begin{equation}
		np  >  C  \kappa^4 \mu^2 r^2 \log q. \label{eq:pAltMin}
	\end{equation}
Then, w.p. at least $1 - 3/n^3$, the iterates $\U^{(t)}$ of AltMin \cite{lowrank_altmin} satisfy $\SE_F(\U^{(t+1)},\Ustar)  \leq 0.25 \SE_F(\U^{(t)},\U^*).$

Consequently, if $p_{tot}  > \frac{C  \kappa^4 \mu^2 r^2 \log q}{n} \log({1}/{\epsilon})$, then, after $T=C\log(1/\eps)$ iterations, $\SE_F(\U^{(T)},\Ustar) \le \eps$.
\label{altmin_thm}
\end{corollary}

We prove this in Appendix \ref{altmin_thm_proof}.  The proof is an easy corollary of Lemmas \ref{lem:init}, \ref{lem:BerrLmma}, \ref{lem:Berr_implic} and \ref{lem:BIncohLmma}.
Comparing \eqref{eq:pAltMin}  to Theorem 2.5 of \cite{lowrank_altmin}, we observe that the sample complexity for AltMinComplete has reduced in $r$ from $r^{4.5} \log r$ to $r^2$.

\subsection{Improved Guarantee for Smooth AltMin \cite{hardt2014understanding}}
We will reduce the overall sample complexity by decreasing $p_{init}$ to $O(r^2)$ below.  Formally,
\begin{corollary} Let $r, \eps > 0$. Let $\Xstar$ be a symmetric $n \times n$ matrix. Then, for $T = O(\log(n/\eps))$,  the output of SAltLS satisfies $\SE_F (\U,\U^*) \leq \eps$ with probability $9/10$, provided that sampling probability $p \geq p_{init} + p_{LS}$, where
\begin{equation}
	p_{\mathrm{init}} \geq C{ \kappa^2 r^2 \mu  \log n}/{n },
	\quad  p_{\mathrm{LS}}  \geq C{ \mu^2\kappa^2 r^2 \log(n/\eps)\log^2 n}/{n}.
	\label{eq:HardtNew}
\end{equation}
\end{corollary}
This is a direct corollary of Theorem 6.1 of \cite{hardt2014understanding} combined with our initialization lemma, Lemma \ref{lem:init}.

\section{Simulation Results} \label{sims}
\label{sec:results}
\begin{figure}
	\begin{center}
		\begin{subfigure}{0.35 \linewidth}		
			\includegraphics[width=0.99\linewidth]{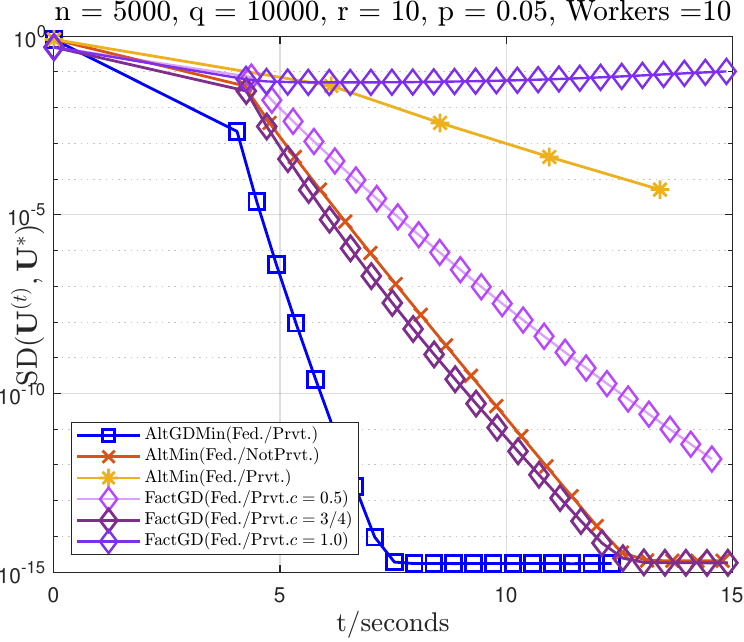}
			\caption{Federated, $r=10$.}
			\label{fig:r10Fed}
		\end{subfigure}
		\begin{subfigure}{0.35 \linewidth}
			\includegraphics[width=0.99\linewidth]{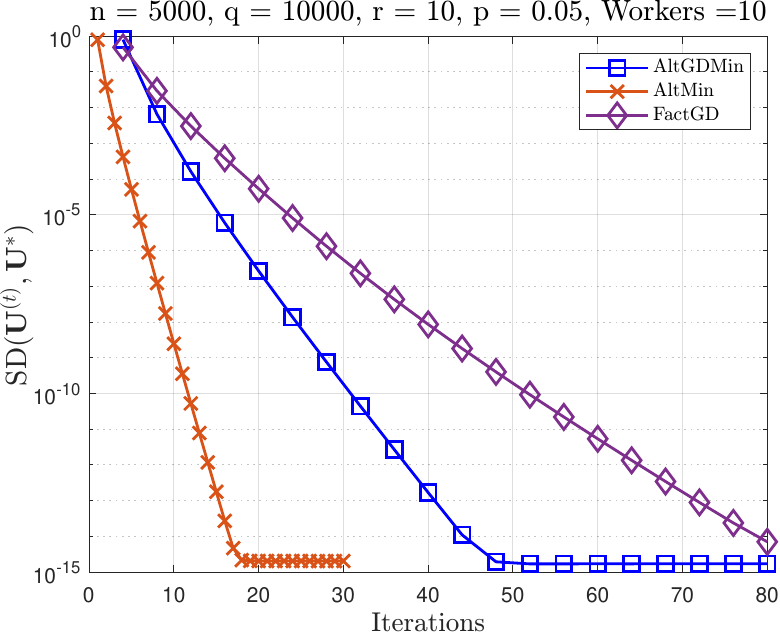}
			\caption{Error against iterations, $r=10$.}
		\end{subfigure}		
\\ \vspace{0.1in}
		\begin{subfigure}{0.35 \linewidth}	
			\includegraphics[width=0.99\linewidth]{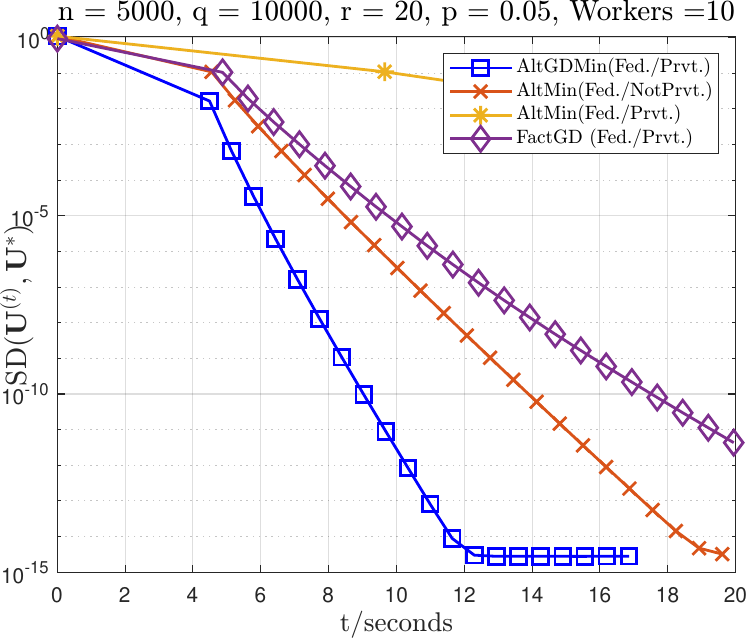}
			\caption{Federated, $r=20$.}
		\end{subfigure}
		\begin{subfigure}{0.35 \linewidth}
			\includegraphics[width=0.99\linewidth]{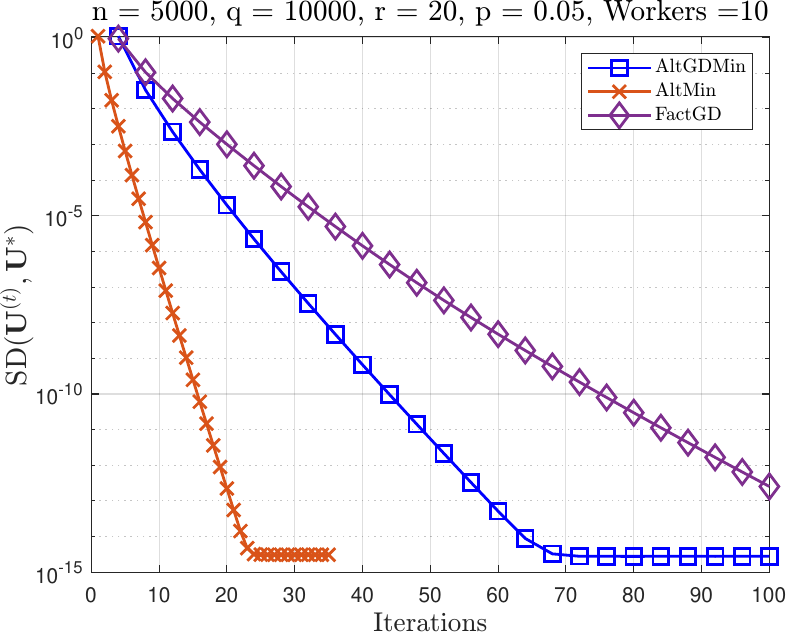}
			\caption{Error against iterations, $r=20$.}
		\end{subfigure}
		
	\caption{\small\sl{Figures (a),(c) compare federated implementations of AltGDmin (proposed), FactGD and AltMin. The results match what our theory (sufficient conditions) predicts. AltGDmin is the fastest due to its lowest communication-efficiency and due to all three having comparable computation cost.
Figures (b), (d) compare iteration complexity numerically. The results match theory once again. 
Two versions of AltMin are compared: AltMin (Fed/Not-Prvt) uses exact LS solution for updating both $\B$ and $\U$. AltMin (Fed/Prvt) uses multiple gradient descent iterations to solve the LS problem for updating $\U$.
In (a), we also compare FactGD with three choices of step size. See experiments' description for details.%
%
}}
\label{figs}
	\vspace{-0.25in}
	\end{center}
\end{figure}


\subsubsection {Experimental Setup} We plot averaged subspace distance at the iteration $t$ against the average time taken until iteration $t$, with the averages being computed over 100 Monte Carlo runs. The averaging is over the observed entries which are generated uniformly at random.
The matrix $\Xstar = \Ustar \Bstar$ was generated once: we let $\Ustar$  be an orthonormalized $n \times r$ random Gaussian matrix and $\Bstar$ an $r \times q$ random Gaussian matrix.  We used the `parfor' loop in MATLAB to distribute the computation across 10 workers, with each worker being an individual core of a multi-core CPU (13-th Gen. Intel Core i7 with 32GB RAM and 16 cores). To compute the left-singular vectors needed for the initialization of federated algorithms, we used the federated power method and performed 15 power iterations. Also, for both AltMin and AltGDMin, we do not sample-split, i.e., we run each iteration of the algorithms on all observed entries.
\subsubsection{Step-size and other parameters} For  FactGD (Centralized), we used the code provided by the authors of that work \cite{rpca_gd}.  The step size was set $\eta_{\text{FactGD}} = pc/\lVert \Y \rVert$, as also done in the authors' own implementation of their algorithm. Setting $c = 0.75$ showed the fastest convergence for our simulations (see Fig. \ref{fig:r10Fed}). For AltMin (Fed./NotPrvt.) and AltMin(Fed./Prvt.), we wrote our own MATLAB implementation; the latter uses GD to solve the LS problem for updating $\U$. We set the number of GD iterations for solving each LS problem to 10 with step size $\eta = p/\| \Y \|^2$, which showed the fastest convergence. Note that the $\U$-update least-squares problem is convex, and the chosen step-size is an upper bound on the Lipschitz constant of the expected gradient of the objective function at $t=0$ \cite{abbasi2023fast}.   AltGDMin is also implemented with step-size $\eta = p/\| \Y \|^2$, which approximates the step size choice suggested by Theorem \ref{thrm}, since $ ||\E[ \Y]||^2 = p^2 \sigmaxTwo$. AltMin and AltGDMin were both initialized by lines 2-4 of Algorithm \ref{altgdmin}. The  incoherence parameter $\mu$ was estimated as $\hat \mu = \argmin_{\mu} \rV \u^{(0)j} \rV \leq \mu \sqrt{r/n} \, \forall j \in [n]$.
\subsubsection{Observations}   Federated algorithms are faster than their centralized counterparts, except AltMin (Fed. Prvt.). The proposed algorithm AltGDMin (Fed.) is faster than all benchmark methods, especially at the higher value of $r=20$. Specifically,
{\em AltGDMin (Fed.) converges to $\Ustar$ in approximately 12 seconds, compared to nearly 20 seconds for the second fastest AltMin (Fed. NotPrvt.). } We interpret/explain these observations and the benchmark methods below.
%


%
\subsubsection {AltGDMin} AltGDMin (Fed.) has faster convergence because of low communication complexity and low time complexity at the center.  The upstream communication complexity, $\min(\sum_{k \in \S_\ell} \nabla_{\U} f(\U,\b_k) ,nr) = nr$, is low because the nodes sum the column-wise gradients 
\subsubsection {FactGD} $\nabla_{\U} f = \nabla_{\X}\B^\intercal \in \mathbb{R}^{n \times r}$ and $\B\B^\intercal \in \mathbb{R}^{r \times r}$  are computed at the nodes and transmitted upstream. The total upstream communication cost is $nr + r^2$, higher than  the $nr$ communication cost of AltGDMin. 
$\B$ is updated locally (GD iteration and normalizing) at the nodes but this requires two data exchanges with the center. This is because the gradient of $\B$, $(\B\B^\intercal - \U^\intercal\U)\B$, with respect to the norm balancing term, $\lV \U^\intercal\U -\B\B^\intercal \rV_F^2$, cannot be computed at the nodes. The nodes compute  and transmit $\sum_{k \in \mathcal{S}_{\ell}} \b_{k}\b_{k}^\intercal$,  which are summed at the center to form $\B\B^\intercal$, and transmitted back to the nodes. The center also computes and transmits $\U^\intercal\U$. Then, at the nodes, the partial gradient $(\B\B^\intercal - \U^\intercal\U)\sum_{k \in \mathcal{S}_{\ell}} \b_{k}\b_{k}^\intercal$ is computed, followed by column-normalizing. The federated and centralized implementations of FactGD do not use `for' loops at all, either at the nodes or at the center. But, FactGD is slower than AltMin and AltGDMin because of its higher communication complexity, $O(nr + r^2)$ and 2 data exchanges, compared to  1 data exchange and $O(nr)$, $O(qr)$ communication complexity for AltGDMin and AltMin (Fed.NotPrvt.), respectively.
\subsubsection {AltMin} AltMin (Fed./NotPrvt.)  is slower than AltGDMin (Fed.) because the $n$ $\U$-update LS problems are solved sequentially at the center with complexity $\lvert \Omega \rvert r^2$, compared to the $nr^2$ complexity of computing the QR decomposition for AltGDMin. Also, AltMin (Fed./NotPrvt.) is not private because the nodes communicate the updated $\b_k^{(t+1)}$ to the center. For AltMin(Fed./Prvt.), the $\U$-update LS problems are solved by multiple gradient descent iterations at the node. 
While private, the GD version of AltMin is slow because of the communication overhead of transmitting/receiving the gradients several times ($\log(1/\eps)$ times for $\eps$-accuracy) in each iteration. AltMin (Cntrl.) is slower than federated methods because both $\U$, $\B$ LS  problems are solved sequentially by the closed form solution at the center, but it is still the fastest centralized algorithm overall, possibly because of $\log(1/\eps)$ iteration complexity, which is lower than that of other methods.

\section{Conclusions}\label{conclude}
In this work we developed and analyzed the Alternating GD and Minimization  (AltGDmin) algorithm for solving the low rank matrix completion (LRMC) problem. The design of AltGDmin is motivated by a federated setting. 
%
Using our results (sample and iteration complexity bounds) we argued that, in a federated setting, AltGDmin is the most communication-efficient solution. It is also one of the two fastest private solutions and has the second smallest sample complexity. In addition, we were able to use our lemmas to prove an improved sample complexity guarantee for AltMin, which is the fastest centralized solution for LRMC.


\appendices \renewcommand\thetheorem{\Alph{section}.\arabic{theorem}}
\newcommand\gradUtilde{\stackrel{\sim}{\smash{\gradU}\rule{0pt}{1.01ex}}}

\section{Proof of the Corollary \ref{cor_noisy} (noisy LRMC)} \label{proof_noisy}
We can prove the following by directly modifying the noise-free case proof.
\begin{claim}
		Let the observed matrix $\Y = \X^*_{\Omega} + \W_{\Omega}$ be corrupted by additive noise $\W$ such that $\lvert \W_{jk} \rvert \leq \eps_{noise} \lvert \X^*_{jk} \rvert$. Assuming also that $p \geq C \kappa^6 r^2 \mu \log q \log(1/\eps)/n$ and  $\eps_{noise} \leq 1/(\sqrt r \kappa^3)$,  then, after $T = \kappa^2\log(1/\eps_{noise})$ iterations, $\SEF(\U^{(T)},\U^*) \leq \kappa^2 \eps_{noise} \sqrt{r}$.

		\label{thm:Nsy}
\end{claim}

\subsection{Proof}
Restricting $\eps_{noise} \leq 1/(\sqrt{r} \kappa^3)$ (as done in the statement of Theorem \ref{thm:Nsy}) ensures that the initialization bound in Lemma \ref{lem:init} differs only by a constant factor (see Lemma \ref{lem:nsyInit}).

With  noise $\W_\Omega$, the Least-Squares update \eqref{eq:LS_Updt} changes to
\begin{equation*}
	\b_k - \g_k
 = \F_k^{-1} (\d_k +\U_k^\intercal\w_k), \forall k \in [q],	\label{eq:nsyLS_Updt}
\end{equation*}
where $\w_k$ is the $k$-th column of $\W_\Omega$. Consequently, the bound in \eqref{eq:B-GExprsn} changes to $\|\B - \G\|_F \le \|\F^{-1}\|  (\|\D\|_F  + \lV \W_{\B}\rV_F)$, where $\W_B = \sum_k \U_k^\top \w_k \e_k^\top$. In Lemma \ref{lem:nsyTrms}, we bound $\lV \W_{\B} \rV_F  \leq 2 \sqrt r \eps_{noise} p \sigmax$, ensuring that $\|\B - \G\|_F \le \max(\deltatold,  2\sqrt{r} \eps_{noise}) \sigmax$. Assuming $2 \sqrt{r} \eps_{noise} \leq \deltatold$,  Lemmas \ref{lem:BerrLmma} and \ref{lem:Berr_implic} continue to hold. Because $\lV \U_k^\intercal \w_k \rV \leq 2 \eps_{noise} p \lV \b_k^* \rV_2$ (Lemma \ref{lem:nsyTrms}), $ \lV (\U_k^\intercal\U_k)^{-1}\U_k^\intercal\w_k \rV_2 \leq  2\eps_{noise} \lV \b^*_k\rV_2$, which, together with $\lV \b_k \rV_2 \leq 1.1 \sigmax \mu \sqrt{r/q}$ (Lemma \ref{lem:BIncohLmma}), bounds the noisy least-squares update $\lV \tilde \b_k \rV_2 \leq  (1.1 + 2\eps_{noise})\mu \sigmax \sqrt{r/q}$, thereby proving $\mu$-incoherence in the noisy setting.

With noise, the gradient is $\gradUtilde  = \gradU -  \W_{\U}$, where $\W_\U = \W_{\Omega}\B^\intercal$. We bound $\lV \W_{\U} \rV_F \leq 2\eps_{noise}\sqrt{r}p\sigmaxTwo$ in Lemma \ref{lem:nsyTrms}, and assume $\eps_{noise} \leq \delta^{(t-1)}/\sqrt{r}$ so that the bound in Lemma \ref{lem:Ugrad} continues to hold. To prove row-incoherence in the noisy setting, the proof of Lemma \ref{lem:rowIncoh} needs to be only slightly modified. \eqref{eq:rowGradjDvtn} still applies and \eqref{eq:ujBnd} has the additional row-vector term $\eta p \w^j \B^{*\intercal}$, which can be bounded $\lV \w^j \B^{*\intercal} \rV_2 \leq \lV \w^j \rV_2  \lV \B^* \rV \leq \eps_{noise} \lV \x^{*j} \rV_2  \sigmax \leq \eps_{noise} \lV\u^{*j} \rV  \sigmaxTwo$. This term contributes $\eps_{noise}\lV \u^{*j} \rV $in the final bound, that is, $\lV \tilde \u^{+j} \rV_2 \leq (1 - 0.15/\kappa^2)\lV  \u^j \rV_2 + 2 \lV \u^{*j} \rV_2  + \eps_{noise}\lV \u^{*j} \rV$, which ensures $\kappa^2 \mu$ incoherence for the noisy gradient update  $\tilde \u^{+j}$, since $\eps_{noise} \leq 1$.

The proof details are the same as those for proving Theorem \ref{thrm} given earlier in Sec. \ref{thrm_proof}.


\subsection{Noisy case lemmas} 
Our assumption $\lvert \W_{jk} \rvert \leq \eps_{noise}  \lvert \X^*_{jk} \rvert$ also implies that
\begin{equation}
\max\big(\frac{\lV \w_k \rV}{ \lV \x^*_k\rV},\frac{\lV \w^j \rV}{\lV \x^{*j} \rV}\big) \le \eps_{noise}, \text{ and } |\W_\jk| \le \eps_{noise} \mu^2 \sqrt{r/q} \sqrt{r/n} \sigmax
\label{assuw1}
\end{equation}
This fact is used in proving both lemmas below.

\begin{lemma}
	Let the observed matrix $\Y = \X^*_{\Omega} + \W_{\Omega}$ be corrupted by additive noise $\W$ such that $\lvert \W_{jk} \rvert \leq \eps_{noise} \lvert \X^*_{jk} \rvert$. Assuming also that $p \geq C \kappa^6 r^2 \mu \log q/n$ and  $\eps_{noise} \leq 1/\sqrt r \kappa^3$,  then, $\SEF(\U^{(00)},\U^*) \leq c/\kappa^2$, where $\U^{(00)} \in \mathbb{R}^{n \times r}$ are the left-singular vectors of $\Y$.
	\label{lem:nsyInit}
\end{lemma}

\begin{proof}
By Wedin's sin theta theorem (Frob norm version),
	\begin{equation}
		\SE_F(\U^{(00)}, \Ustar) \le C \frac{\|\Y - p\Xstar\|_F}{\sigmin}
		\le  C \frac{\|\Xstar_\Omega- p\Xstar\|_F + \|\W_\Omega \|_F}{p \sigmin}
		\le C  \sqrt{\frac{\kappa^2 \mu r^2 \log n}{p}}   + \frac{\|\W_\Omega \|_F}{ p \sigmin},
		\label{eq:SD_Noisy}
	\end{equation}
where we have used Lemma 3.21 of \cite{spectral_init_review}, with $p \geq  C \kappa^6 r^2 \mu \log q/n$,  to bound $(\lV \X_\Omega^* - p \X^*\rV_F / p\sigmin) \leq c/\kappa^2$. By using matrix-Bernstein inequality, we can show that $\|\W_\Omega\|_F \le 2 \eps_{noise} \sqrt r p \sigmax$ whp.
By the assumed upper bound on $\eps_{noise}$, we can then argue that  $\SEF(\U^{(00)},\U^*) \leq c/\kappa^2$. The rest of the proof is exactly the same as in the noiseless case (Lemma \ref{lem:init}).
\end{proof}

\begin{lemma}
	\label{lem:nsyTrms}
	Assume that the observed matrix $\Y = \X^*_{\Omega} + \W_{\Omega}$ is corrupted by additive noise $\W$ such that $\lvert \W_{jk} \rvert \leq \eps_{noise} \lvert \X^*_{jk} \rvert$. Assuming also $\lV \u^j \rV \leq \mu_{\u} \sqrt{r/n}$, where $\mu_u = \kappa^2 \mu$, $\lV \b_k \rV \leq \sigmax \mu\sqrt{r/q}$ and $\sigma_{\max}(\B) \leq C\sigmax$. For $\eps_{noise} \leq 1/(\sqrt{r}\kappa^3)$,
	\ben
	\item $\lV \W_{B} - \mathbb{E}[\W_{\B}] \rV_F \leq \eps_{noise} \sqrt r p\sigmax$, w.p. greater than $1 - \exp(\log 2q - np / (\mu_{\u}^2 r^2))$ 
	\item $\lV \W_{\B} \rV_F \leq 2\sqrt{r} \eps_{noise} p \sigmax$, w.p. same as 1).
	\item $\lV \U_k^\intercal \w_k \rV_2 \leq 2 \eps_{noise} p \lV \b^*_k \rV_2$, w.p. same as 1) 
	\item $\lV \W_{\U} - \mathbb{E}[\W_\U] \rV_F \leq \eps_{noise} \sqrt r p \sigmaxTwo$, w.p. greater than $1 - \exp(\log 2q - np / (\mu^2 r^2))$ 
	\item $\lV \W_\U \rV_F \leq 2 \eps_{noise} \sqrt r p \sigmaxTwo$, w.p. same as 4).
	\een
\end{lemma}

\begin{proof}
All terms are bounded using \eqref{assuw1} and the matrix Bernstein inequality.
\end{proof}

\section{Proof of Corollary \ref{altmin_thm} (Improved guarantee for AltMin for LRMC)}\label{altmin_thm_proof}

	In section 4.2 of \cite{lowrank_altmin}, it is shown that the least-squares updates for $\U \in \mathbb{R}^{n \times r}$ and $\B \in \mathbb{R}^{r \times q}$ are equivalent to the following QR-based updates:
	\begin{align}
		 \B^{(t)} &= \argmin_{ \B} \lV \Y - \U^{(t)}  \B \rV_F^2, \,\,\,   \B^{(t)\intercal} \qreq \V^{(t)}\R_{\B} \label{eq:BLS_Updt}\\
		\widetilde \U^{(t+1)} &=  \argmin_{\widetilde \U} \lV \Y^\intercal - \B^{(t)} \widetilde \U \rV_F^2, \,\,\, \widetilde \U^{(t+1)^\intercal} \qreq \U^{(t+1)}\R_\U. \label{eq:ULS_Updt}
	\end{align}
	\subsubsection{$\mathbf{B}$-update \eqref{eq:BLS_Updt}\,/\,$\SE_F(\mathbf V^{(t)},\mathbf V^*)$ Bound}
	Note that the $\B$-update in \eqref{eq:BLS_Updt} is exactly the same as the $\B$-update for AltGDMin. Therefore, we use Lemma \ref{lem:BerrLmma} to bound $\SE_F(\V^{(t)},\V^*)$. Let $\P \coloneq \mathbf{I} - \V^*\V^{*\intercal}$. Then,
	\begin{align} \SE_F(\V^{(t)},\V^*)  &= \lV  \P \V^{(t)} \rV_F = \lV  \P \B^{\intercal} \R_{B}^{-1} \rV_F \leq \lV \P \B^\intercal  \rV_F \lV \R_{B}^{-1}\rV \leq {\lV (\B-\G)^\intercal \rV_F}/{\sigma_{\min}(\B)}, \label{eq:SEF_AltMin}
	\end{align}
	where $\P\B^\intercal = \P(\G^\intercal + (\B-\G)^\intercal) = \P(\B - \G)^\intercal$, because $\G^\intercal = \X^{*\intercal}\U^\intercal$  and the column-span of $\X^{*\intercal}$ is orthogonal to the span of $\P$.
By Lemma \ref{lem:BerrLmma}, $\lV  \B - \G \rV_F \leq \sigmin \deltatold/4$. By Lemma \ref{lem:Berr_implic},  assuming $\deltatold \leq 1/(4\kappa)$, $\sigma_{\min}(\B) \geq \sigmin/2$.  Substituting these bounds in \eqref{eq:SEF_AltMin} and recalling that $\deltatold \coloneq \SE_F(\U^{(t)},\U^*)$, w.p. greater than $1 - \exp(\log q - pn/(r^2 \mu^2\kappa^4))$,
	\begin{equation}
		\SE_F(\V^{(t)},\V^*) \leq  \frac{1}{2} \SE_F(\U^{(t)},\U^*). \label{eq:altMinVBnd}
	\end{equation}
	\subsubsection{Incoherence of $\mathbf{V}^{(t)}$} Recall that $\widetilde \B^\intercal \qreq \V\R$. Then, the $k$-th row, $\lV \v^k \rV_2  \leq {\lV \b_k\rV_2}/{\sigmin(\B)} \leq C \kappa \mu \sqrt{r/q}$, where the last inequality follows from Lemmas \ref{lem:Berr_implic} and  \ref{lem:BIncohLmma}. Thus, $\V$ is $\kappa \mu$-incoherent.

	\subsubsection{$\mathbf{U}$-Update \eqref{eq:ULS_Updt}\,/\,$\SE_F(\mathbf{U^{(t+1)}},\mathbf{U}^*)$ Bound and Incoherence of $\mathbf{U}^{(t+1)}$} 
The two AltMin steps are symmetric, so arguments analogous to the above help show that
		\begin{equation}
		\SE_F(\U^{(t+1)},\U^*) \leq  \frac{1}{2} \SE_F(\V^{(t)},\U^*). \label{eq:altMinUBnd}
	\end{equation}

	
	Combining \eqref{eq:altMinVBnd} and $\eqref{eq:altMinUBnd}$, we have $\SE_F(\U^{(t+1)},\U^*) \leq \SE_F(\U^{(t)},\U^*)/4 $. For initialization we use the few lines of our Algorithm 1 and Lemma \ref{lem:init}.  
Consequently, AltMinComplete  \cite{lowrank_altmin} initialized as given in the first few lines of our Algorithm 1  needs $T =\log(1/\eps)$ iterations for $\SE_F(\U^{(t)},\U^*) \leq \eps $.

\bibliographystyle{IEEEtran}	
\bibliography{../../bib/tipnewpfmt_kfcsfullpap,ref}

\begin{thebibliography}{10}
\providecommand{\url}[1]{#1}
\csname url@samestyle\endcsname
\providecommand{\newblock}{\relax}
\providecommand{\bibinfo}[2]{#2}
\providecommand{\BIBentrySTDinterwordspacing}{\spaceskip=0pt\relax}
\providecommand{\BIBentryALTinterwordstretchfactor}{4}
\providecommand{\BIBentryALTinterwordspacing}{\spaceskip=\fontdimen2\font plus
\BIBentryALTinterwordstretchfactor\fontdimen3\font minus
  \fontdimen4\font\relax}
\providecommand{\BIBforeignlanguage}[2]{{%
\expandafter\ifx\csname l@#1\endcsname\relax
\typeout{** WARNING: IEEEtran.bst: No hyphenation pattern has been}%
\typeout{** loaded for the language `#1'. Using the pattern for}%
\typeout{** the default language instead.}%
\else
\language=\csname l@#1\endcsname
\fi
#2}}
\providecommand{\BIBdecl}{\relax}
\BIBdecl

\bibitem{allrtnPpr}
A.~A. Abbasi, S.~Moothedath, and N.~Vaswani, ``Fast federated low rank matrix
  completion,'' in \emph{2023 59th Annual Allerton Conference on Communication,
  Control, and Computing (Allerton)}.\hskip 1em plus 0.5em minus 0.4em\relax
  IEEE, 2023, pp. 1--6.

\bibitem{matcomp_candes}
E.~J. Candes and B.~Recht, ``Exact matrix completion via convex optimization,''
  \emph{Found. of Comput. Math}, no.~9, pp. 717--772, 2008.

\bibitem{lowrank_altmin}
P.~Netrapalli, P.~Jain, and S.~Sanghavi, ``Low-rank matrix completion using
  alternating minimization,'' in \emph{Annual ACM Symp. on Th. of Comp.
  (STOC)}, 2013.

\bibitem{matcomp_first}
M.~Fazel, ``Matrix rank minimization with applications,'' \emph{PhD thesis,
  Stanford Univ}, 2002.

\bibitem{optspace}
R.~Keshavan, A.~Montanari, and S.~Oh, ``Matrix completion from a few entries,''
  \emph{IEEE Trans. Info. Th.}, vol.~56, no.~6, pp. 2980--2998, 2010.

\bibitem{mc_luo}
R.~Sun and Z.-Q. Luo, ``Guaranteed matrix completion via non-convex
  factorization,'' \emph{IEEE Trans. Info. Th.}, vol.~62, no.~11, pp.
  6535--6579, 2016.

\bibitem{rmc_gd}
Y.~Cherapanamjeri, K.~Gupta, and P.~Jain, ``Nearly-optimal robust matrix
  completion,'' \emph{ICML}, 2016.

\bibitem{pr_mc_reuse_meas}
C.~Ma, K.~Wang, Y.~Chi, and Y.~Chen, ``Implicit regularization in nonconvex
  statistical estimation: Gradient descent converges linearly for phase
  retrieval, matrix completion and blind deconvolution,'' in \emph{Intl. Conf.
  Machine Learning (ICML)}, 2018.

\bibitem{lowrank_altmin_no_kappa}
M.~Hardt and M.~Wootters, ``Fast matrix completion without the condition
  number,'' in \emph{Conf. on Learning Theory}, 2014.

\bibitem{lafferty_lrmc}
Q.~Zheng and J.~Lafferty, ``Convergence analysis for rectangular matrix
  completion using burer-monteiro factorization and gradient descent,''
  \emph{arXiv preprint arXiv:1605.07051}, 2016.

\bibitem{rpca_gd}
X.~Yi, D.~Park, Y.~Chen, and C.~Caramanis, ``Fast algorithms for robust pca via
  gradient descent,'' in \emph{Neur. Info. Proc. Sys. (NeurIPS)}, 2016.

\bibitem{hardt2014understanding}
M.~Hardt, ``Understanding alternating minimization for matrix completion,'' in
  \emph{2014 IEEE 55th Annual Symposium on Foundations of Computer
  Science}.\hskip 1em plus 0.5em minus 0.4em\relax IEEE, 2014, pp. 651--660.

\bibitem{fastmc}
P.~Jain and P.~Netrapalli, ``Fast exact matrix completion with finite
  samples,'' in \emph{Conf. on Learning Theory}, 2015, pp. 1007--1034.

\bibitem{dist_lrmc}
L.~W. Mackey, A.~Talwalkar, and M.~I. Jordan, ``Distributed matrix completion
  and robust factorization,'' \emph{J. Mach. Learn. Res.}, vol.~16, no.~1, pp.
  913--960, 2015.

\bibitem{dist_lrmc_2}
C.~Teflioudi, F.~Makari, and R.~Gemulla, ``Distributed matrix completion,'' in
  \emph{2012 ieee 12th international conference on data mining}.\hskip 1em plus
  0.5em minus 0.4em\relax IEEE, 2012, pp. 655--664.

\bibitem{qling_etal}
Q.~Ling, Y.~Xu, W.~Yin, and Z.~Wen, ``Decentralized low-rank matrix
  completion,'' in \emph{2012 IEEE International Conference on Acoustics,
  Speech and Signal Processing (ICASSP)}, 2012, pp. 2925--2928.

\bibitem{qling_etal2}
A.-Y. Lin and Q.~Ling, ``Decentralized and privacy-preserving low-rank matrix
  completion,'' \emph{Journal of the Operations Research Society of China},
  vol.~3, no.~2, pp. 189--205, 2015.

\bibitem{mardani2013decentralized}
M.~Mardani, G.~Mateos, and G.~Giannakis, ``Decentralized sparsity-regularized
  rank minimization: Algorithms and applications,'' \emph{IEEE Trans. Sig.
  Proc.}, 2013.

\bibitem{fed_lrmc_1}
Z.~Li, B.~Ding, C.~Zhang, N.~Li, and J.~Zhou, ``Federated matrix factorization
  with privacy guarantee,'' \emph{Proceedings of the VLDB Endowment}, vol.~15,
  no.~4, 2021.

\bibitem{fed_lrmc_2}
V.~W. Anelli, Y.~Deldjoo, T.~Di~Noia, A.~Ferrara, and F.~Narducci,
  ``User-controlled federated matrix factorization for recommender systems,''
  \emph{Journal of Intelligent Information Systems}, vol.~58, no.~2, pp.
  287--309, 2022.

\bibitem{fed_lrmc_3_byz}
X.~He, Q.~Ling, and T.~Chen, ``Byzantine-robust stochastic gradient descent for
  distributed low-rank matrix completion,'' in \emph{2019 IEEE Data Science
  Workshop (DSW)}.\hskip 1em plus 0.5em minus 0.4em\relax IEEE, 2019, pp.
  322--326.

\bibitem{lrpr_gdmin}
S.~Nayer and N.~Vaswani, ``Fast and sample-efficient federated low rank matrix
  recovery from column-wise linear and quadratic projections,'' \emph{IEEE
  Trans. Info. Th.}, Feb. 2023, arXiv:2102.10217 (Feb. 2021).

\bibitem{lrpr_gdmin_2}
N.~Vaswani, ``Efficient federated low rank matrix recovery via alternating gd
  and minimization: A simple proof,'' \emph{IEEE Trans. Info. Th.}, 2024,
  arXiv: 2306.17782.

\bibitem{lrpr_it}
S.~Nayer, P.~Narayanamurthy, and N.~Vaswani, ``Provable low rank phase
  retrieval,'' \emph{IEEE Trans. Info. Th.}, March 2020.

\bibitem{lrpr_best}
S.~Nayer and N.~Vaswani, ``Sample-efficient low rank phase retrieval,''
  \emph{IEEE Trans. Info. Th.}, Dec. 2021.

\bibitem{versh_book}
R.~Vershynin, \emph{High-dimensional probability: An introduction with
  applications in data science}.\hskip 1em plus 0.5em minus 0.4em\relax
  Cambridge University Press, 2018, vol.~47.

\bibitem{Chen_2021}
\BIBentryALTinterwordspacing
Y.~Chen, Y.~Chi, J.~Fan, and C.~Ma, ``Spectral methods for data science: A
  statistical perspective,'' \emph{Foundations and Trends® in Machine
  Learning}, vol.~14, no.~5, p. 566–806, 2021. [Online]. Available:
  \url{http://dx.doi.org/10.1561/2200000079}
\BIBentrySTDinterwordspacing

\bibitem{nashed1968decomposition}
M.~Nashed, ``A decomposition relative to convex sets,'' \emph{Proceedings of
  the American Mathematical Society}, vol.~19, no.~4, pp. 782--786, 1968.

\bibitem{zarantonello1971projections}
E.~H. Zarantonello, ``Projections on convex sets in hilbert space and spectral
  theory: Part i. projections on convex sets: Part ii. spectral theory,'' in
  \emph{Contributions to nonlinear functional analysis}.\hskip 1em plus 0.5em
  minus 0.4em\relax Elsevier, 1971, pp. 237--424.

\bibitem{twf}
Y.~Chen and E.~Candes, ``Solving random quadratic systems of equations is
  nearly as easy as solving linear systems,'' in \emph{Neur. Info. Proc. Sys.
  (NeurIPS)}, 2015, pp. 739--747.

\bibitem{abbasi2023fast}
A.~A. Abbasi, S.~Moothedath, and N.~Vaswani, ``Fast federated low rank matrix
  completion,'' in \emph{2023 59th Annual Allerton Conference on Communication,
  Control, and Computing (Allerton)}.\hskip 1em plus 0.5em minus 0.4em\relax
  IEEE, 2023, pp. 1--6.

\bibitem{spectral_init_review}
Y.~Chen, Y.~Chi, J.~Fan, and C.~Ma, ``Spectral methods for data science: A
  statistical perspective,'' \emph{arXiv preprint arXiv:2012.08496}, 2020.

\bibitem{trefethen}
L.~N. Trefethen and D.~Bau, \emph{Numerical linear algebra}.\hskip 1em plus
  0.5em minus 0.4em\relax Siam, 2022, vol. 181.

\end{thebibliography}

\end{document}